\ifdefined\pdfminorversion\pdfminorversion=7\fi
\RequirePackage{fix-cm}
\PassOptionsToPackage{table}{xcolor}
\documentclass{valence}

\usepackage{amsmath,amsfonts,bm}

\def\eqref#1{equation~\ref{#1}}

\def\ceil#1{\lceil #1 \rceil}
\def\floor#1{\lfloor #1 \rfloor}
\def\1{\bm{1}}

\def\eps{{\epsilon}}

\DeclareMathAlphabet{\mathsfit}{\encodingdefault}{\sfdefault}{m}{sl}
\SetMathAlphabet{\mathsfit}{bold}{\encodingdefault}{\sfdefault}{bx}{n}

\newcommand{\E}{\mathbb{E}}

\newcommand{\R}{\mathbb{R}}

\usepackage{hyperref}
\usepackage{xurl}
\usepackage[utf8]{inputenc} %
\usepackage[T1]{fontenc}
\usepackage{lmodern}
\AtBeginDocument{\normalfont\DeclareFontShape{T1}{lmr}{bx}{sc}{<->ssub*lmr/m/sc}{}}
\usepackage{csquotes}

\usepackage{booktabs}       %
\usepackage{amsfonts}       %
\usepackage{nicefrac}       %
\usepackage{microtype}      %

\usepackage{multirow}
\usepackage{makecell}
\usepackage[table]{xcolor}
\usepackage{adjustbox}
\usepackage{array}
\usepackage{pifont}

\usepackage{amsmath}
\usepackage{amssymb}
\usepackage{mathtools}
\usepackage{amsthm}
\usepackage{thmtools}
\usepackage{dsfont}
\usepackage{algorithm}
\usepackage{float}
\usepackage{algpseudocode}

\makeatletter
\newcounter{algorithmicH}
\let\oldalgorithmic\algorithmic
\renewcommand{\algorithmic}{%
  \stepcounter{algorithmicH}%
  \oldalgorithmic}%
\renewcommand{\theHALG@line}{ALG@line.\thealgorithmicH.\arabic{ALG@line}}
\makeatother

\usepackage{subcaption}

\usepackage[capitalize,noabbrev]{cleveref}
\crefname{appendix}{Appendix}{Appendices}
\Crefname{appendix}{Appendix}{Appendices}

\definecolor{headerblue}{RGB}{230,238,250}
\definecolor{subheaderblue}{RGB}{242,246,252}
\definecolor{lightgrayrow}{RGB}{248,249,250}

\definecolor{ink}{RGB}{32,48,66}
\definecolor{muted}{RGB}{85,93,105}
\definecolor{reference}{RGB}{239,243,248}
\newcommand{\val}[2]{#1\,{\fontsize{6.2}{7}\selectfont\color{muted}$\pm$#2}}

\declaretheorem[numberwithin=section]{theorem}
\makeatletter
\def\thmt@autorefsetup{\@xa\def\csname\thmt@envname autorefname\@xa\endcsname\@xa{\thmt@thmname}}
\makeatother
\declaretheorem[sibling=theorem]{proposition,lemma,corollary}
\declaretheorem[sibling=theorem,style=definition]{definition,assumption}

\let\E\relax\let\R\relax\let\eps\relax
\let\floor\relax\let\ceil\relax
\DeclareMathOperator{\E}{\mathbb{E}}
\DeclareMathOperator{\ent}{H}

\DeclareMathOperator{\bigO}{\mathcal{O}}
\DeclareMathOperator{\indic}{\mathds{1}}
\DeclarePairedDelimiter{\abs}{\lvert}{\rvert}

\DeclarePairedDelimiter{\floor}{\lfloor}{\rfloor}
\DeclarePairedDelimiter{\ceil}{\lceil}{\rceil}
\newcommand{\eps}{\varepsilon}
\newcommand{\R}{\mathbb R}

\newcommand{\methodname}{\textsc{SeedER}}

\usepackage[most]{tcolorbox}
\definecolor{darkblue}{RGB}{5, 64, 124}

\tcbset{
  theoremstyle/.style={
    colback=headerblue,
    colframe=darkblue,
    coltitle=white,
    fonttitle=\bfseries\sffamily,
    before title={\colorlet{darkblue}{white}},
    before upper={\colorlet{darkblue}{darkblue}},
    boxrule=0.8pt,
    arc=4pt,
    left=6pt,
    right=6pt,
    top=6pt,
    bottom=6pt
  }
}

\hypersetup{colorlinks=true, linkcolor=darkblue, urlcolor=darkblue, citecolor=darkblue}
\title{SeedER: Seed-Expand-Retrieve for Efficient Knowledge Graph Retrieval}

\author[1,2,3\star]{Hamed Shirzad}
\author[1,2]{Frederik Wenkel}
\author[1,2]{Dominique Beaini}
\author[3\dagger]{Danica Sutherland}
\author[1,2\dagger]{Emmanuel Noutahi}

\affiliation[1]{Valence Labs, Montr\'eal, QC, Canada}
\affiliation[2]{Recursion, Salt Lake City, UT, USA}
\affiliation[3]{Department of Computer Science, University of British Columbia, Vancouver, BC, Canada}

\contribution[\star]{Work done during an internship at Valence Labs.}
\contribution[\dagger]{Joint supervision.}

\abstract{

Knowledge graphs (KGs) offer a rich representation for relational knowledge, but their irregular structure makes retrieval challenging: ego-graph expansion grows rapidly, and dense embedding methods struggle with multi-hop compositional queries. Several approaches use LLM agents to explore the KG, analyze candidate nodes, and decide where to explore next. While expressive, these approaches can incur substantial computational and memory costs. On the other hand, we show theoretically that dense embeddings precomputed for graph nodes, even with augmented structure and neighborhood-aware features, can require embedding dimensions comparable to the size of the graph to answer families of knowledge graph queries. This limitation can be overcome with query-adaptive embeddings under certain conditions, and there are graph neural network (GNN) variants that can do so. However, processing the whole graph with a GNN can also incur substantial memory and computational costs, and it requires dense ground-truth labels indicating whether each node answers the query. Straightforward $k$-hop selection around anchor nodes can also be problematic: small $k$ limits the nodes we can see, and even with small $k$ values such as three and four, the $k$-hop subgraph can grow substantially. In this work, we devise a new approach using Graph Transformers and reinforcement learning that is much less demanding in memory and computation, can run on a CPU at inference time as a first-stage retriever, and requires feedback on only a small number of nodes at a time during training. We show that this method is competitive with methods that fine-tune LLMs to retrieve information from KGs. We call our method \methodname{} (Seed-Expand-Retrieve), and position it primarily as a first-stage retriever that can run with modest computational resources.

}

\begin{document}

\maketitle

\section{Introduction}

Knowledge graphs (KGs) are a central substrate for storing and reasoning over structured knowledge in domains such as biomedicine, scientific discovery, and enterprise search.
Nodes represent entities with textual descriptions, while typed edges encode relations that support complex, multi-hop reasoning. Despite this expressivity, KGs pose a fundamental challenge for retrieval: relevant answers to a query are often not lexically similar to the query itself, but are instead reachable only through a sequence of relational constraints. 
If we ask which drugs treat Alzheimer's disease via the cholinergic pathway, answer nodes such as \emph{Donepezil} or \emph{Galantamine} share no terms with the query, but are uniquely identified by the relational path connecting them. Effective retrieval over KGs requires reasoning \emph{over graph structure}, not merely over text \citep{stark,primekg}.

Most existing retrieval systems adopt dense embedding-based approaches that score nodes independently by similarity to the query \citep{lewis2020retrieval,karpukhin2020dense}. While effective for some queries, such methods struggle on compositional graph queries that require chaining multiple relations: the prior example is \emph{Alzheimer's} (\textsc{disease}) $\to$ \emph{ACHE} (\textsc{gene}) $\to$ \emph{cholinergic signaling} (\textsc{pathway}) $\to$ \emph{Donepezil} (\textsc{drug}).
Augmenting dense retrieval with $k$-hop graph expansion helps, but suffers from rapid candidate-set explosion and introduces substantial noise. A GNN-based query-independent pre-computed embedding can map local structure into node embeddings but still score nodes through a single embedding comparison, inheriting the same representational bottleneck. At the other extreme, agent-based graph exploration methods~\citep{graphflow,graph_r1} explicitly traverse the KG, but call LLMs or similar models sequentially over individual nodes, making them prohibitively expensive for large-scale, first-stage retrieval.

We argue that the core difficulty of KG retrieval lies in a mismatch between \emph{global} embedding-based scoring and the \emph{local} structure of graph reasoning. Compositional queries are naturally resolved by a sequence of local decisions -- selecting which neighboring nodes to explore next -- yet dense retrievers attempt to approximate the entire composition in a single embedding space. We show that this mismatch is fundamental: there exist reasonable families of KGs and queries for which dense retrieval requires per-node feature embeddings of size linear in the size of the graph.

A query-adapted embedding such as QA-GNN \citep{qa-gnn} can bypass this limitation~--~but running a GNN on the whole KG for each query is infeasible for very large KGs.
There have been solutions proposed to this issue such as finding seed nodes from the query and only exploring a $k$-hop neighborhood around the seed nodes \citep{stark}.
We show, however, that $k$-hop subgraphs even for $k = 3$ or $4$
can be nearly the size of the whole graph,
while $k = 1$ or $2$ misses many queries that need deeper, multi-hop reasoning.

Another major issue is that while benchmark datasets might contain ground truth answer nodes for queries,
these dense labels are not typically available in practice for novel queries on a new KG:
we probably have to get feedback on a small set of nodes at a time,
limited e.g.\ by an LLM's context window,
particularly in common KGs where each node is a large document \citep{stark}.
We would like to be able to train in scenarios with this sparse feedback,
even if the usual public benchmarks do not require it.

Motivated by these observations, we introduce a novel retrieval framework we call \textbf{\methodname}~(\textbf{Seed}-\textbf{E}xpand-\textbf{R}etrieve).
\methodname{} operates in three stages.
First, we find \emph{seed} nodes, such as through cosine similarity or named entity recognition.
This step is common to many retrieval methods.

Next, we \emph{expand} from those seed nodes into a neighborhood of a reasonable size with nodes that are likely to be relevant to the query, with an algorithm we call \textsc{K-hop-with-filtering}.
This step is a simple but, to our knowledge, novel algorithm
that does not require learning.
It produces a candidate subgraph
which makes intermediate nodes available for the next step's reasoning,
while allowing it to search over a graph of a few hundred relevant nodes,
rather than millions of mostly-irrelevant ones.

Finally, we \emph{retrieve} a relevant subset from this expanded neighborhood
with a specialized variant of an Exphormer \citep{exphormer},
a sparse Graph Transformer (GT),
that can adapt to a given query.
This GT has two heads:
one selects a subset of relevant nodes, say twenty of them,
and the other ranks the selected nodes.
We do selection with Plackett--Luce algorithms from \citet{PL-ranking}
trained with a \textsc{Reinforce} variant \citep{reinforce,rloo,grpo,drgrpo},
accommodating feedback on a small set of proposed nodes;
the ranking head is trained with BPR \citep{bpr}.

Running the retrieval model on the bounded subgraph produced by expanding from the seeds,
rather than on the full KG,
gives enormous improvements to scalability.
Most of our experiments,
even though they are run on datasets with as many as two million nodes
and approximately twenty million edges, using the counts in \cref{tab:skb_statistics},
can be run on a typical consumer laptop
which would not be able to run even one step of an agent-based selection workflow.
The drawback of this subselection is that if the previous step missed an answer node, we cannot recover it.
We will show in experiments that, with the size of graphs we use,
most answer nodes are retained;
it would be possible in future work, however,
to enrich this step with the ability to get nodes outside the initial input,
as done e.g.\ by Graph-R1 \citep{graph_r1}.

We show theoretically that our query-conditioned message passing approach
can handle relation tracing in knowledge graphs
with node features of size $\bigO(\log \abs V)$,
rather than the $\Omega(\abs V \log \abs V)$
we show is required for dense retrievers.
We also decompose how the steps contribute to final retrieval quality.
In experiments,
\methodname{} achieves better retrieval with compact candidate sets than dense and graph-augmented baselines on STARK-PRIME, STARK-MAG, and STARK-AMAZON, making it a strong and efficient first-stage retriever for knowledge-intensive reasoning pipelines.

\begin{figure}[t]
    \centering
    \includegraphics[width=\linewidth]{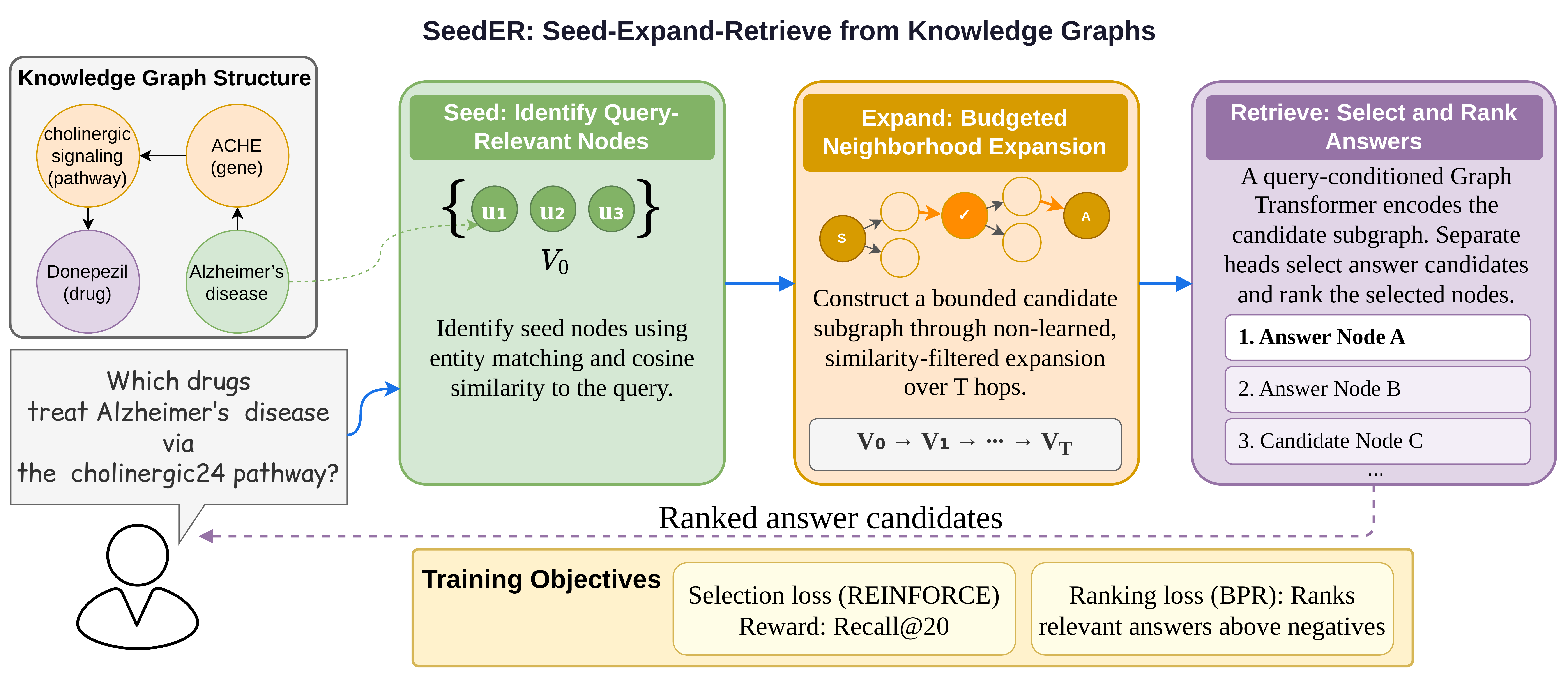}
    \caption{
    Overview of SeedER (Seed-Expand-Retrieve). Entity-based and dense retrieval identify query-relevant seed nodes, which are expanded through non-learned, similarity-filtered neighborhood exploration to construct a bounded candidate subgraph. A query-conditioned Graph Transformer then selects a compact set of answer candidates and ranks them using separate prediction heads. Training combines a REINFORCE objective that maximizes expected Recall@20 with an auxiliary pairwise ranking loss (BPR).
    }
    \label{fig:seeder_diagram}
\end{figure}

\section{Background and Problem Setting}
\label{sec:preliminaries}

We consider a typed knowledge graph $G=(V,E)$, where $V$ is a set of nodes and
$E \subseteq V \times \mathcal{R} \times V$ is a set of directed edges labeled by relation types
$\mathcal{R}$.
Each node $v\in V$ is associated with a textual description $x_v$.
For example, the precision-medicine knowledge graph PrimeKG~\citep{primekg} contains roughly $|V|=129\text{K}$ nodes of ten types, including \textsc{disease}, \textsc{drug}, \textsc{gene} and \textsc{pathway}, together with $8.1\text{M}$ directed typed relations. These relations include biomedical links such as \textsc{associations} between \textsc{disease} and \textsc{gene}, \textsc{indications} between \textsc{drug} and \textsc{disease}, and \textsc{memberships} between \textsc{gene} and \textsc{pathway}. 

Let $\mathcal{Q}=\{q_1,\dots,q_{|\mathcal{Q}|}\}$ denote a set of evaluation queries;
each query $q\in\mathcal{Q}$ is associated with a %
ground-truth set
$A(q)\subseteq V$ of relevant nodes in the KG.  

A \emph{retrieval model} returns an ordered list
$R(q) = [v_1,\dots,v_k]$
with $k \ll \abs V$.
One goal could be for $R(q) = A(q)$,
i.e.\ fully answering the query with the retrieval model.
This is a difficult task; it is also usually not the actual goal of knowledge graph interaction.
Instead, we want systems which can use the content of knowledge graphs as input for their own reasoning processes,
such as in retrieval-augmented generation,
where the goal is to find a small $R(q) \supseteq A(q)$ to give useful context to a textual generative model:
a \emph{first-stage retriever}.

To evaluate this aim, we use metrics such as
Hit@$k$ -- the portion of the time a $k$-node $R(q)$ contains any answer nodes --
and Recall@$k$
-- the average portion of answers identified in a $k$-node $R(q)$.
Using @any means the Hit or Recall for all the identified nodes.
To evaluate rankings,
we use the Mean Reciprocal Rank (MRR; \citealp{mrr}).
\Cref{app:metrics} has more details.

\section{Related Work}
Dense retrieval is a common approach for first-stage retrieval:
select candidates based on embedding similarity to the query~\citep{karpukhin2020dense,reimers-2019-sentence-bert,zhang2025qwen3}.
Although effective for many problems, it misses answers whose relevance depends on multi-hop relations.
Subsequent reranking can improve ordering, but cannot recover answers absent from the retrieved set~\citep{stark,beir,dang2013twostage,zamani2022stochastic}.

Graph-based retrieval methods explicitly exploit relational structure. G-Retriever and GRAG retrieve textual subgraphs for question answering; GraphRAG uses graph-based indexing and community summaries to construct context~\citep{gretriever,grag,graphrag}. HippoRAG uses graph propagation for associative retrieval, and GFM-RAG learns a graph neural retriever~\citep{hipporag,gfmrag}. Agentic approaches such as Graph-R1 and GraphFlow formulate retrieval as sequential graph interaction or exploration~\citep{graph_r1,graphflow}.

Further discussion of related work appears in Appendix~\ref{app:related_work}.

\section{Failures of dense retrieval}
\label{sec:challenges}

Dense bi-encoders embed queries and nodes into a shared vector space and perform retrieval via inner products. While highly effective for standard document retrieval, this paradigm struggles to capture relational structure in knowledge graphs, where relevance is often determined not by local textual similarity, but by multi-hop relational dependencies between nodes.

A natural approach is to enrich node representations with structural information. For example, \citet{stark} augment node text with descriptions of incident relations. While this improves over purely single-document information retrieval, such representations remain fundamentally local: they cannot faithfully encode multi-hop dependencies beyond immediate neighborhoods. Increasing textual context can also degrade embedding quality: longer inputs introduce noise, reducing the effectiveness of similarity-based retrieval.
They can also require expensive recomputation of embeddings for large portions of a graph from even small changes in the KG.

No matter whether these embeddings come from textual summaries of nodes,
structure-aware embeddings computed by a graph neural network,
or combine multiple sources for each node,
any method that decides relevance solely based on a fixed pair of query and node embeddings is subject to a severe capacity bottleneck.

\begin{tcolorbox}[theoremstyle,title={Informal version of \cref{thm:hard-argmax,thm:iterative-tracing}}]
Consider a graph where nodes are related through a fixed number $k \ge 3$ of relation types,
where we wish to trace queries of length $L$ at least logarithmic in the number of nodes.

Suppose a
graph-independent function assigns scores
from a query embedding and a node embedding,
such that the answer to a query is scored among the top $B$ nodes.
Then the node embeddings must have at least
$\Omega\bigl(\frac{\abs V}{B} \log \frac{\abs V}{B} \bigr)$ bits per node on average.

On the same problem,
there is a graph-independent query-conditioned message-passing network,
with $L$ layers and width $\bigO(L k)$,
which assigns score 1 to correct answers
and score 0 to other nodes.
\end{tcolorbox}

The formal statements and proof are in \cref{app:dense-lb}.
From results related to classic properties of expander graphs \citep{expander-survey},
being able to resolve queries based on a single node embedding
requires that embedding to contain nearly all information about the graph,
regardless of the algorithm used to create these embeddings.
Even if the node embeddings are calculated with an arbitrarily complicated function of the whole graph,
and the comparison between a query and a node embedding
is done by some other arbitrarily complicated function,
this information-theoretic lower bound is unavoidable.
Tracing relations inevitably requires knowledge of the whole graph.

\begin{figure}[t]
    \centering
    \includegraphics[width=13.5cm]{figures/RL_Multihop-Multi_Hop_Sampling.pdf}
    \caption{
    Illustration of $k$-hop bounded subgraph construction. Starting from seed node(s), the algorithm selects the most promising nodes in each hop with the given budget. The new nodes are selected based on the similarity score from the original node, the edge between the nodes, and the target node's similarity. A node not selected in a given round can still be a candidate in future rounds.
    }
    \label{fig:rl_multihop_sampling}
\end{figure}

The positive result, on the other hand,
shows that even a very simple message-passing neural network
can trace the desired relations.
This network's parameters work across graphs and relation sequences;
given different subgraphs from the seed and expand steps for different queries,
it still succeeds,
as long as the input graph contains the required reasoning path.
In realistic problems,
a network would need to learn to do this reasoning 
from language,
but this result shows a clear separation between dense retrieval and query-conditioned message passing in terms of their ability to do search on graphs.

\section{Practical graph seeding and local expansion}
Dense embeddings may be insufficient for directly retrieving the final answers, but they can still provide useful anchors for graph search. For instance, in a biomedical knowledge graph, a query about genes involved in a disease may be difficult to answer in one shot, especially with non-specialized embeddings. Nevertheless, a dense retriever can often easily identify the node representing the disease itself; local relations around that node can provide routing signals toward relevant genes. Thus expanding around seed nodes can provide a sensible compact subgraph of candidates, to later be filtered by a learned model.

We combine named entity recognition (NER) and cosine similarity to identify initial seed nodes.
Our NER first identifies entities mentioned in the query by matching its token spans to node names and aliases in the KG. We prefer longer matches, order ambiguous matches by query--node cosine similarity, and retain at most $k_0$ distinct nodes. If fewer than $k_0$ nodes are found, we fill the remaining slots with the closest nodes by cosine similarity. This gives a core set $V_0$ with $k_0$ distinct nodes. Entity matches anchor the search to what the query explicitly mentions; cosine retrieval on token embeddings provides anchors when a name is absent, or expressed differently.

Given this seed set,
a primary challenge is that naive $k$-hop expansion grows rapidly. In graphs with hub nodes, even a two- or three-hop neighborhood can cover a large fraction of the graph; this expansion behavior is illustrated for STARK-PRIME in \cref{fig:multi-hop-expansion-results}(a), and analyzed further in \cref{app:k-hop-expansion}.

To control this growth, we use a simple frontier-filtered expansion procedure, which we call \textsc{K-hop-with-filtering}.
Starting from seed nodes, the method repeatedly considers only the one-hop neighbors of the current frontier $\mathcal F$.
Each candidate neighbor $u$ is scored by the mean of the query $q$'s similarity to the source node $v \in \mathcal F$, the connecting relation $r$, and the candidate node $u$ itself:
$$\displaystyle
\mathrm{score}(u)
=
\max_{v \xrightarrow{r} u,\;v \in \mathcal F}
\frac{\mathrm{sim}(q,v)+\mathrm{sim}(q,r)+\mathrm{sim}(q,u)}{3}
.$$
At hop $h$, the method keeps only the top $b_h$ scoring candidates, adds them to the retrieved set, and uses them as the next frontier.
The algorithm is illustrated in \cref{fig:rl_multihop_sampling}.

Despite its simplicity, this filtered expansion on STARK-PRIME improves substantially over pure cosine-similarity retrieval in Hit@any and Recall@any across different retrieval budgets (\cref{fig:multi-hop-expansion-results}).
Compared to dense retrieval,
as well as containing more answer nodes for a given budget,
it retains a connected subgraph around each seed node,
rather than scattered nodes which are often totally isolated from one another in a candidate subgraph.

\section{Learning to select from a fixed candidate graph}

\textsc{K-hop-with-filtering} supplies a bounded candidate graph.
A query-conditioned encoder can then use its multi-hop structure to score all candidates
before selecting the final output set.
The intermediate nodes needed for reasoning remain available to the encoder,
without necessarily having to be selected by the retriever.

\begin{tcolorbox}[theoremstyle,title={\cref{prop:selection-oracle,prop:selection-ranking,prop:selection-bpr,prop:selection-feedback}: Candidate coverage and ranking}]
For an answer set $A$ and candidate set $C$, the best possible Recall@$B$ is
$\min\{B,\abs{A\cap C}\}/\abs A$.
Recall loss decomposes into candidates missing from the subgraph, the output budget, and selection error.
The selection error can be minimized by
regression to the target $1 / \abs A$ for answer nodes and $0$ for non-answer nodes,
by pairwise ranking loss between answers and non-answers,
or by reinforcement learning from feedback on whether a node is an answer.
\end{tcolorbox}
\Cref{sec:appendix:selection} gives the full statements and proofs.
Thus, there are several approaches that can work well for ranking.
Our motivation for RL is that applications may provide judgments for only a small number
of proposed nodes, rather than an exhaustive answer set for each query.
A policy can propose a set, receive feedback on those nodes, and update using its observed utility. Recall rewards additionally require the total answer count $|A(q)|$. When that count is unavailable, we can instead train with $|S\cap A(q)|/20$, the number of valid results in the proposed set $S$ divided by the fixed output budget of 20. This reward requires only judgments on the proposed nodes; supervised losses can also use the judged nodes or pairs. In a matched five-seed Qwen comparison, this alternative gives mean test Recall@20 within 0.58 percentage points of recall-reward training on all three datasets, with higher variation across seeds (\cref{app:fixed_denominator_reward}).
\Cref{prop:selection-feedback} formalizes the policy update and the effect of our group-mean baseline.

\begin{figure}[t]
    \centering
    \includegraphics[width=\linewidth]{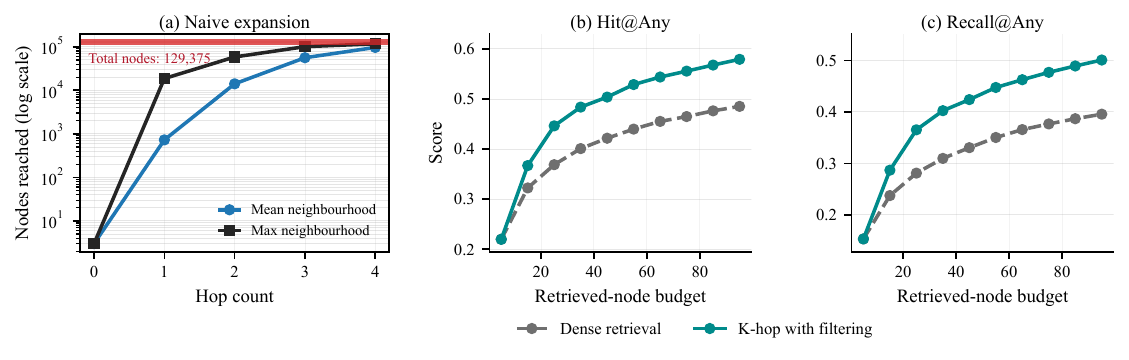}
    \vspace{-0.5cm}
    \caption{
    Expansion rates of naive multi-hop neighborhoods and retrieval performance of \textsc{K-hop-with-filtering} under different retrieval budgets
    on the STARK-PRIME dataset (see \cref{sec:experiments} for more setup details).
    For each query, we initialize expansion from the top three nodes by cosine similarity, and measure the mean and largest neighborhood size.
    }
    \label{fig:multi-hop-expansion-results}
    \vspace{-0.1cm}
\end{figure}

\paragraph{Query-Adaptive Graph Encoder.}
We use an adapted version of Exphormer~\citep{exphormer} to encode the candidate subgraph. Its sparse attention retains the KG relations and allows communication between different branches of the retrieved neighborhood, without all-pairs attention. 
Exphormers can use
expander edges (distributed connections that give additional routes between real nodes),
a virtual node (a hub, which we initialize from the query, connected to all other nodes),
or both.
(These connections are used only by the encoder.)
As described by \citet{exphormer}, all versions are expressive and can learn long-range dependencies, but optimization can vary;
we thus consider this choice a model hyperparameter.

We also inject the query directly into the node and relation features, using their text embeddings and cosine similarities to the query. Each node receives a learned positional embedding of the hop number at which it was first retrieved; seed nodes are hop number zero. This encoding aligns well with the objective of our task.
We use a learned gate to combine this positional encoding with the original node embeddings.

The encoder is evaluated once per query on the candidate subgraph, giving representations $h_v^{(L)}(q)$.

\paragraph{Selection and Ranking.}
Two small heads share these representations.
A policy head $g_\theta$ selects $k$ nodes out of those nodes whose type is eligible to be an answer.
A scoring head $\phi_\theta$ then assigns an order to the selected nodes.
During training, we sample $M$ subsets without replacement from the policy and reward their answer coverage. We use REINFORCE~\citep{reinforce} with Plackett-Luce gradient estimation \citep{PL-ranking} for the policy head, using the mean reward for this query as a baseline.
The scoring head uses a BPR ranking loss~\citep{bpr},
which is scaled relative to the policy loss as
$
\mathcal L=\mathcal L_{\mathrm{RL}}+
\lambda_{\mathrm{BPR}}\mathcal L_{\mathrm{BPR}},
$
for a scalar hyperparameter $\lambda_{\mathrm{BPR}}$.
The policy learns which nodes to retain, and BPR supplies supervised gradients for their ordering and the shared encoder. At inference time, we take the top-$k$ policy scores and sort those nodes by $\phi_\theta$. Architecture, loss details, and pseudocode for both training and inference are given in \cref{app:exphormer,app:policy_objective,app:main_algorithms}.

\section{Experiments} \label{sec:experiments}
\paragraph{Experimental setup.}
We evaluate \methodname{} on STARK-PRIME, STARK-MAG, and STARK-AMAZON~\citep{stark}, covering biomedical, academic, and product retrieval. We use the official splits and report Hit@1, Hit@5, MRR, and Recall@20. For each query, \methodname{} constructs a candidate graph with at most 500 real nodes and returns at most 20 eligible answers. We evaluate MiniLM-L6-v2\footnote{\url{https://huggingface.co/sentence-transformers/all-MiniLM-L6-v2}} \citep{minilm,reimers-2019-sentence-bert}, ada-002~\citep{openai2022ada002}, and Qwen embeddings \citep{zhang2025qwen3} and report means and sample standard deviations over five training seeds. Each run is trained for a fixed number of epochs, and its checkpoint is selected by validation Recall@20. Dataset descriptions and metric definitions are in \cref{app:datasets,app:metrics}; configurations and evaluation details are in \cref{app:experiment_details}.

\paragraph{Main results.}
\Cref{tab:main_results} compares \methodname{} with dense retrieval, graph-based retrieval, and LLM-based methods. With MiniLM, \methodname{} improves all four metrics over the baselines in the same encoder block, reaching Recall@20 of 0.649, 0.629, and 0.519 on PRIME, MAG, and AMAZON. With ada-002, Recall@20 increases to 0.735, 0.710, and 0.630, although reranked baselines retain slightly higher scores on the other AMAZON metrics. Qwen achieves 0.794, 0.751, and 0.623: close to GraphFlow's 0.797 on PRIME and above its 0.572 and 0.362 on MAG and AMAZON. These are comparisons of complete systems whose backbones and retrieval protocols differ. LLM baselines include an additional reranking step performed by an LLM.

\paragraph{Effect of candidate graphs.}
To examine how candidate construction affects retrieval, we train MiniLM models on candidate graphs prepared with Qwen while keeping MiniLM features and model settings fixed within each dataset. Across five seeds, Recall@20 improves by 11.06, 8.65, and 8.85 percentage points on PRIME, MAG, and AMAZON, respectively, with improvements in all four metrics. This shows that stronger candidate graphs can benefit the same downstream model. The comparison changes graph membership and retrieval structure together; the full results and matched configurations are in \cref{app:minilm_candidate_graphs}.

\begin{table}[t]
\centering
\caption{Retrieval performance across encoders and STARK datasets. H@k denotes Hit@k and R@20 denotes Recall@20; higher is better. \methodname{} rows report five-seed means with sample standard deviations beneath them. Bold identifies the highest mean per metric within models compared in each section.}
\label{tab:main_results}
\label{tab:prime_main_results}
\label{tab:graphflow_comparison}
\begingroup\fontsize{7.7}{9}\selectfont
\setlength{\tabcolsep}{0pt}\renewcommand{\arraystretch}{1.0}
\begin{tabular}{@{}w{c}{16pt}p{\dimexpr\textwidth-16pt-294pt-12pt\relax}*{4}{w{r}{24.5pt}}@{\hspace{6pt}}*{4}{w{r}{24.5pt}}@{\hspace{6pt}}*{4}{w{r}{24.5pt}}@{}}
\toprule
&&\multicolumn{4}{c}{\textbf{PRIME}}&\multicolumn{4}{c}{\textbf{MAG}}&\multicolumn{4}{c}{\textbf{AMAZON}}\\
\cmidrule(lr){3-6}\cmidrule(lr){7-10}\cmidrule(l){11-14}
LM&Method&H@1&H@5&MRR&R@20&H@1&H@5&MRR&R@20&H@1&H@5&MRR&R@20\\
\midrule

\multirow{3}{*}{\rotatebox[origin=c]{90}{\fontsize{4.5}{6}\selectfont\color{ink}RoBERTa}} & DPR & 0.045 & 0.219 & 0.124 & 0.301 & 0.105 & 0.352 & 0.213 & 0.421 & 0.153 & 0.479 & 0.302 & 0.445 \\
 & ANCE & 0.065 & 0.157 & 0.111 & 0.165 & 0.220 & 0.365 & 0.291 & 0.353 & 0.310 & 0.511 & 0.407 & 0.420 \\
 & QA-GNN & 0.089 & 0.214 & 0.147 & 0.296 & 0.129 & 0.390 & 0.291 & 0.470 & 0.266 & 0.500 & 0.378 & 0.521 \\
\cmidrule(l){2-14}\addlinespace[1pt]
\multirow{9}{*}{\rotatebox[origin=c]{90}{\fontsize{7}{8}\selectfont\color{ink} MiniLM-L6-v2}} & Dense retrieval & 0.101 & 0.218 & 0.161 & 0.259 & 0.203 & 0.358 & 0.273 & 0.344 & 0.306 & 0.525 & 0.411 & 0.417 \\
 & G-Retriever & 0.099 & 0.232 & 0.162 & 0.321 & 0.206 & 0.371 & 0.277 & 0.366 & 0.306 & 0.526 & 0.408 & 0.416 \\
 & SubgraphRAG & 0.099 & 0.220 & 0.156 & 0.262 & 0.203 & 0.359 & 0.273 & 0.345 & 0.306 & 0.525 & 0.408 & 0.414 \\
 & Beam search & 0.099 & 0.249 & 0.167 & 0.331 & 0.204 & 0.387 & 0.281 & 0.392 & 0.306 & 0.528 & 0.404 & 0.429 \\
 & A* search & 0.099 & 0.250 & 0.167 & 0.334 & 0.204 & 0.387 & 0.282 & 0.392 & 0.306 & 0.529 & 0.404 & 0.430 \\
 & PPR & 0.088 & 0.244 & 0.160 & 0.320 & 0.191 & 0.395 & 0.279 & 0.404 & 0.280 & 0.524 & 0.391 & 0.435 \\
 & PPR+MMR & 0.088 & 0.257 & 0.162 & 0.318 & 0.191 & 0.397 & 0.279 & 0.399 & 0.278 & 0.524 & 0.392 & 0.430 \\
 & $k$-hop with filtering & 0.101 & 0.248 & 0.167 & 0.342 & 0.205 & 0.391 & 0.285 & 0.398 & 0.306 & 0.525 & 0.411 & 0.435 \\
 & SeedER & \shortstack[r]{\textbf{0.357}\\[-2pt]{\fontsize{5.8}{6}\selectfont\color{muted}$\pm$0.012}} & \shortstack[r]{\textbf{0.595}\\[-2pt]{\fontsize{5.8}{6}\selectfont\color{muted}$\pm$0.008}} & \shortstack[r]{\textbf{0.460}\\[-2pt]{\fontsize{5.8}{6}\selectfont\color{muted}$\pm$0.011}} & \shortstack[r]{\textbf{0.649}\\[-2pt]{\fontsize{5.8}{6}\selectfont\color{muted}$\pm$0.009}} & \shortstack[r]{\textbf{0.430}\\[-2pt]{\fontsize{5.8}{6}\selectfont\color{muted}$\pm$0.004}} & \shortstack[r]{\textbf{0.621}\\[-2pt]{\fontsize{5.8}{6}\selectfont\color{muted}$\pm$0.005}} & \shortstack[r]{\textbf{0.515}\\[-2pt]{\fontsize{5.8}{6}\selectfont\color{muted}$\pm$0.004}} & \shortstack[r]{\textbf{0.629}\\[-2pt]{\fontsize{5.8}{6}\selectfont\color{muted}$\pm$0.003}} & \shortstack[r]{\textbf{0.356}\\[-2pt]{\fontsize{5.8}{6}\selectfont\color{muted}$\pm$0.006}} & \shortstack[r]{\textbf{0.601}\\[-2pt]{\fontsize{5.8}{6}\selectfont\color{muted}$\pm$0.004}} & \shortstack[r]{\textbf{0.464}\\[-2pt]{\fontsize{5.8}{6}\selectfont\color{muted}$\pm$0.005}} & \shortstack[r]{\textbf{0.519}\\[-2pt]{\fontsize{5.8}{6}\selectfont\color{muted}$\pm$0.002}} \\
\cmidrule(l){2-14}\addlinespace[1pt]
\multirow{6}{*}{\rotatebox[origin=c]{90}{\fontsize{7}{8}\selectfont\color{ink} ada-002}} & VSS & 0.126 & 0.315 & 0.214 & 0.360 & 0.291 & 0.496 & 0.386 & 0.484 & 0.392 & 0.627 & 0.504 & 0.533 \\
 & Multi-VSS & 0.151 & 0.336 & 0.235 & 0.381 & 0.259 & 0.504 & 0.369 & 0.508 & 0.401 & 0.650 & 0.516 & 0.551 \\
 & VSS + Claude 3$^{\dagger}$ & 0.178 & 0.369 & 0.263 & 0.356 & 0.365 & 0.532 & 0.442 & 0.484 & \textbf{0.455} & 0.711 & \textbf{0.559} & 0.538 \\
 & VSS + GPT-4$^{\dagger}$ & 0.183 & 0.373 & 0.266 & 0.341 & 0.409 & 0.582 & 0.490 & 0.486 & 0.448 & \textbf{0.712} & 0.557 & 0.554 \\
 & $k$-hop with filtering & 0.120 & 0.364 & 0.220 & 0.475 & 0.169 & 0.561 & 0.301 & 0.584 & 0.185 & 0.636 & 0.388 & 0.513 \\
 & SeedER & \shortstack[r]{\textbf{0.432}\\[-2pt]{\fontsize{5.8}{6}\selectfont\color{muted}$\pm$0.008}} & \shortstack[r]{\textbf{0.676}\\[-2pt]{\fontsize{5.8}{6}\selectfont\color{muted}$\pm$0.006}} & \shortstack[r]{\textbf{0.539}\\[-2pt]{\fontsize{5.8}{6}\selectfont\color{muted}$\pm$0.006}} & \shortstack[r]{\textbf{0.735}\\[-2pt]{\fontsize{5.8}{6}\selectfont\color{muted}$\pm$0.005}} & \shortstack[r]{\textbf{0.527}\\[-2pt]{\fontsize{5.8}{6}\selectfont\color{muted}$\pm$0.008}} & \shortstack[r]{\textbf{0.709}\\[-2pt]{\fontsize{5.8}{6}\selectfont\color{muted}$\pm$0.005}} & \shortstack[r]{\textbf{0.608}\\[-2pt]{\fontsize{5.8}{6}\selectfont\color{muted}$\pm$0.006}} & \shortstack[r]{\textbf{0.710}\\[-2pt]{\fontsize{5.8}{6}\selectfont\color{muted}$\pm$0.003}} & \shortstack[r]{0.433\\[-2pt]{\fontsize{5.8}{6}\selectfont\color{muted}$\pm$0.006}} & \shortstack[r]{0.711\\[-2pt]{\fontsize{5.8}{6}\selectfont\color{muted}$\pm$0.005}} & \shortstack[r]{0.556\\[-2pt]{\fontsize{5.8}{6}\selectfont\color{muted}$\pm$0.004}} & \shortstack[r]{\textbf{0.630}\\[-2pt]{\fontsize{5.8}{6}\selectfont\color{muted}$\pm$0.004}} \\
\cmidrule(l){2-14}\addlinespace[1pt]
\multirow{8}{*}{\rotatebox[origin=c]{90}{\fontsize{7}{8}\selectfont\color{ink} Qwen3-4B + LLMs}} & Dense retrieval & 0.154 & 0.375 & 0.254 & 0.441 & 0.382 & 0.617 & 0.489 & 0.609 & 0.406 & 0.657 & 0.517 & 0.542 \\
 & $k$-hop with filtering & 0.154 & 0.376 & 0.266 & 0.543 & 0.370 & 0.653 & 0.487 & 0.657 & 0.384 & 0.635 & 0.493 & 0.496 \\
 & ToG + LLaMA3$^{\ddagger}$ & 0.219 & 0.340 & 0.266 & 0.338 & 0.120 & 0.141 & 0.127 & 0.068 & 0.042 & 0.062 & 0.053 & 0.026 \\
 & ToG + GPT-4o$^{\ddagger}$ & \textbf{0.533} & 0.637 & 0.578 & 0.544 & 0.267 & 0.567 & 0.397 & 0.480 & 0.276 & 0.517 & 0.391 & 0.258 \\
 & SFT$^{\ddagger}$ & 0.238 & 0.525 & 0.360 & 0.477 & 0.276 & 0.449 & 0.364 & 0.375 & 0.122 & 0.306 & 0.215 & 0.252 \\
 & PRM$^{\ddagger}$ & 0.229 & 0.282 & 0.269 & 0.460 & 0.273 & 0.441 & 0.337 & 0.367 & 0.213 & 0.425 & 0.320 & 0.357 \\
 & GraphFlow$^{\ddagger}$ & 0.514 & 0.721 & \textbf{0.614} & \textbf{0.797} & 0.391 & 0.575 & 0.478 & 0.572 & \textbf{0.479} & 0.650 & 0.555 & 0.362 \\
 & SeedER & \shortstack[r]{0.461\\[-2pt]{\fontsize{5.8}{6}\selectfont\color{muted}$\pm$0.014}} & \shortstack[r]{\textbf{0.743}\\[-2pt]{\fontsize{5.8}{6}\selectfont\color{muted}$\pm$0.005}} & \shortstack[r]{0.583\\[-2pt]{\fontsize{5.8}{6}\selectfont\color{muted}$\pm$0.010}} & \shortstack[r]{0.794\\[-2pt]{\fontsize{5.8}{6}\selectfont\color{muted}$\pm$0.003}} & \shortstack[r]{\textbf{0.563}\\[-2pt]{\fontsize{5.8}{6}\selectfont\color{muted}$\pm$0.004}} & \shortstack[r]{\textbf{0.755}\\[-2pt]{\fontsize{5.8}{6}\selectfont\color{muted}$\pm$0.004}} & \shortstack[r]{\textbf{0.647}\\[-2pt]{\fontsize{5.8}{6}\selectfont\color{muted}$\pm$0.002}} & \shortstack[r]{\textbf{0.751}\\[-2pt]{\fontsize{5.8}{6}\selectfont\color{muted}$\pm$0.002}} & \shortstack[r]{0.468\\[-2pt]{\fontsize{5.8}{6}\selectfont\color{muted}$\pm$0.005}} & \shortstack[r]{\textbf{0.720}\\[-2pt]{\fontsize{5.8}{6}\selectfont\color{muted}$\pm$0.002}} & \shortstack[r]{\textbf{0.581}\\[-2pt]{\fontsize{5.8}{6}\selectfont\color{muted}$\pm$0.004}} & \shortstack[r]{\textbf{0.623}\\[-2pt]{\fontsize{5.8}{6}\selectfont\color{muted}$\pm$0.002}} \\
\bottomrule
\end{tabular}\endgroup

\par\vspace{5pt}
{\footnotesize\raggedright RoBERTa and ada-002 VSS/Multi-VSS/LLM-reranker results are from \citet{stark}. $^{\dagger}$STARK rerankers use a 10\% test subset; the other STARK rows use the full test set. VSS is dense retrieval for ada-002. G-Retriever and SubgraphRAG are retrieval-only baselines. $^{\ddagger}$Agent-based rows use reranking as reported by \citet{graphflow}; they do not use Qwen embeddings.}
\end{table}

\paragraph{Candidate budget and coverage.}
Increasing the candidate budget generally improves retrieval on PRIME (\cref{fig:prime-candidate-budget}). At 1,600 nodes, Recall@20 reaches 0.812 with Qwen, 0.770 with ada-002, and 0.698 with MiniLM, while the output remains limited to 20 nodes. \Cref{fig:oracle_latency_comparison}(a) compares oracle Hit/Recall@any, which are also upper bounds for our algorithm, to the achieved Hit/Recall@20 by the learned policy. For Qwen, candidate-set oracle Recall@any is 0.884, 0.884, and 0.832 across the three datasets, compared with achieved Recall@20 of 0.794, 0.751, and 0.623.

\begin{figure}[t]
    \centering
    \includegraphics[width=\linewidth]{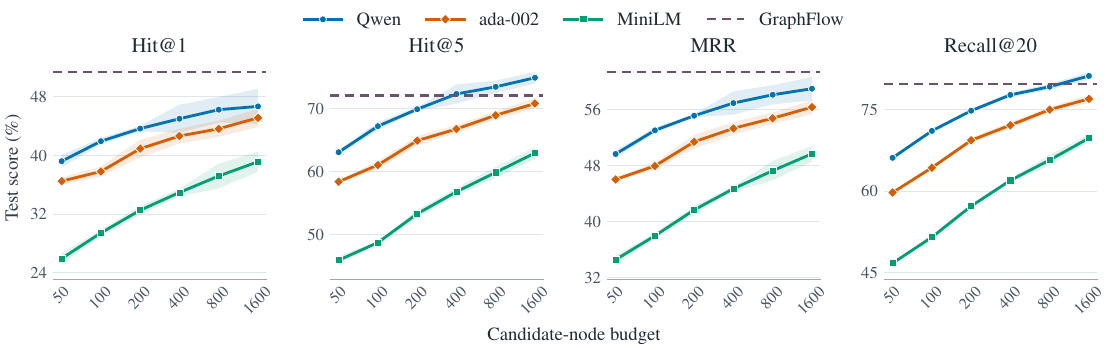}
    \caption{Effect of candidate-node budget on SeedER retrieval performance on PRIME using Qwen, ada-002, and MiniLM embeddings. The budget varies from 50 to 1,600 nodes on a logarithmic axis, while the output budget remains fixed at 20. Lines show mean test performance over five training seeds; shaded bands indicate one sample standard deviation. Dashed lines show reranked GraphFlow results for the PRIME dataset.}
    \label{fig:prime-candidate-budget}
\end{figure}

\subsection{Ablation studies}
\label{sec:ablation}
\Cref{tab:ablation_results} evaluates five-seed Qwen variants on PRIME. Exphormer with expander-only augmentation gives the highest Recall@20 (79.43\%), compared with 78.85\% for virtual-node plus expander augmentation and 78.93\% for virtual-node only. Results are very close across the Exphormer variants, which substantially outperform the simpler GCN and GIN alternatives. 

For the RL policy, removing the policy baseline reduces Recall@20 to 60.50\%; using one trajectory without a baseline reduces it further to 45.14\%. Using the reward of the greedy top-20 policy selection as the REINFORCE baseline yields competitive results, with Recall@20 of 78.18\%.

We further examine training stability across ten MiniLM runs on PRIME in \cref{app:training_stability}. We hold out part of the training data for internal validation and use the original validation split for held-out evaluation. This lets us assess training behavior, variation across seeds, and agreement between validation and held-out performance. The selected checkpoints show low variation across seeds, achieving held-out Recall@20 of $60.87\pm0.46$\%. Internal-validation and held-out metrics are also strongly correlated under this protocol.

\begin{table}[t]
\centering
\caption{Ablations on PRIME with Qwen embeddings. Values are percentages (five-seed mean $\pm$ sample standard deviation); higher is better. The main model uses expander edges without a virtual node. $\Delta$R@20 is the paired mean difference from this reference, in percentage points. }
\label{tab:ablation_results}
\label{tab:prime_results_and_ablations}
\begingroup\fontsize{8}{9.5}\selectfont
\setlength{\tabcolsep}{2.7pt}\renewcommand{\arraystretch}{1.12}
\begin{tabular*}{\textwidth}{@{}c@{\hspace{4pt}}l@{\extracolsep{\fill}}rrrrr@{}}
\toprule
& Variant & H@1 & H@5 & MRR & R@20 & $\Delta$R@20\\
\midrule
 & \textbf{SeedER (main model)} & \val{46.06}{1.38} & \val{\textbf{74.34}}{0.55} & \val{58.30}{0.97} & \val{\textbf{79.43}}{0.27} & --- \\
\cmidrule(l){2-7}
\multirow{3}{*}{\rotatebox[origin=c]{90}{\fontsize{6.3}{7}\selectfont\color{muted}Policy}} & Greedy reward baseline & \val{45.06}{1.10} & \val{72.40}{0.66} & \val{57.09}{0.82} & \val{78.18}{0.43} & -1.25 \\
 & No baseline, 8 trajectories & \val{21.21}{2.98} & \val{49.98}{2.03} & \val{33.88}{2.55} & \val{60.50}{1.96} & -18.93 \\
 & No baseline, 1 trajectory & \val{12.37}{2.90} & \val{34.79}{3.98} & \val{22.50}{3.01} & \val{45.14}{3.90} & -34.28 \\
\cmidrule(l){2-7}
\multirow{4}{*}{\rotatebox[origin=c]{90}{\fontsize{6.3}{7}\selectfont\color{muted}Exphormer}} & Virtual node only & \val{\textbf{46.88}}{1.06} & \val{73.62}{1.28} & \val{\textbf{58.62}}{1.06} & \val{78.93}{0.52} & -0.50 \\
 & Expander only (main model) & \val{46.06}{1.38} & \val{\textbf{74.34}}{0.55} & \val{58.30}{0.97} & \val{\textbf{79.43}}{0.27} & 0.00 \\
 & Virtual node + expander & \val{46.43}{0.95} & \val{73.47}{0.83} & \val{58.26}{0.90} & \val{78.85}{0.54} & -0.58 \\
 & Without hop embeddings & \val{46.12}{0.82} & \val{73.14}{0.43} & \val{57.96}{0.42} & \val{78.86}{0.34} & -0.57 \\
\cmidrule(l){2-7}
\multirow{2}{*}{\rotatebox[origin=c]{90}{\fontsize{6.3}{7}\selectfont\color{muted}GNN}} & GCN & \val{28.19}{0.50} & \val{56.77}{0.58} & \val{40.95}{0.39} & \val{64.48}{0.38} & -14.94 \\
 & GIN & \val{24.33}{0.58} & \val{53.10}{0.69} & \val{37.01}{0.48} & \val{61.80}{0.81} & -17.63 \\
\bottomrule
\end{tabular*}\endgroup

\par\vspace{5pt}
\end{table}

\paragraph{Latency and memory.}
On 20 PRIME validation queries, \methodname{} takes $66.84\pm49.71$\,ms per query and peaks at 0.35\,GiB of allocated GPU memory, compared with $607.00\pm234.55$\,ms and 15.47\,GiB for a three-call Llama3.1-8B \citep{llama3}, on a single node at a time and only up to 16 new tokens generation (\cref{fig:computation_latency_comparison}(b)). The LLM metrics are just a simple baseline for comparing the memory peak and latency; agents might need to think and create tokens sequentially, significantly increasing the required time. The SeedER model has 1.115M learnable parameters; its measurement includes retrieval, bounded graph creation, and ranking with precomputed query embeddings.

\begin{figure}[t]
\centering
\includegraphics[width=\linewidth]{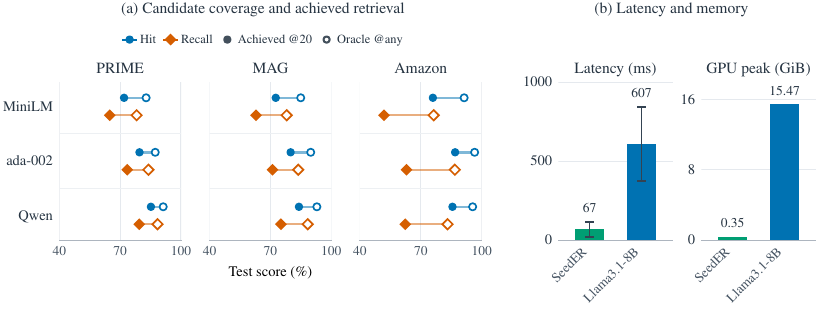}
\caption{Candidate coverage, achieved retrieval, and resource use. (a) Full-test results for the configurations in \cref{tab:main_results}. Circles denote Hit and diamonds denote Recall; open markers show candidate-set oracle @any, and filled markers show achieved @20 (five-seed means, with sample-standard-deviation error bars). Segments show the gap to candidate coverage, which does not impose the 20-output limit. (b) Mean latency and peak allocated GPU memory on 20 PRIME validation queries with three repeats, using an H100 MIG 40\,GB GPU. Latency error bars show sample standard deviations. SeedER uses the virtual-node model with precomputed query embeddings; timing includes retrieval, subgraph generation and reranking. Llama3.1-8B is a three-call generative traversal proxy using preassembled text.}
\label{fig:computation_latency_comparison}
\label{fig:oracle_latency_comparison}
\end{figure}

\section{Conclusion, Limitations, \& Future Work}
\label{sec:conclusion}

We introduced \methodname{}, a method for retrieving from knowledge graphs 
based on bounded subgraph construction and a sampling policy learned using Graph Transformers and \textsc{Reinforce}. By using graph structure directly, \methodname{} can recover relevant nodes that are not textually similar to the query but are reachable through informative relational paths. Our theoretical analysis shows that dense retrieval can face capacity bottlenecks on compositional graph queries, while query-conditioned message passing can answer the same queries with compact node states. Empirically, \methodname{} improves over dense retrieval and heuristic graph expansion across STARK benchmarks, and achieves competitive or superior results to algorithms that use LLMs to search one node at a time. We also show that \methodname{} is compatible with a diverse set of dense encoders, and consistently improves retrieval performance on top of them.

Our method is bounded by the $k$-hop extraction algorithm and would miss nodes that are removed from the search space in this stage.
While we show that most answers are retained,
more expressive algorithms for producing the candidate subgraph
and algorithms that can proactively explore beyond that
are natural areas for future extensions.

\pagebreak

\bibliography{refs}
\bibliographystyle{valence}

\newpage

\appendix
\section*{Appendix}
\crefalias{section}{appendix}
\crefalias{subsection}{appendix}

\section{Extended Related Work}
\label{app:related_work}

\paragraph{First-stage retrieval and reranking.}
A common retrieval paradigm uses a lower-budget \emph{first-stage retriever} to construct a candidate set $R(q)$, with $\abs{R(q)} \ll \abs{V}$, which is then filtered or reordered by a stronger but more expensive \emph{reranker}. Retrieve-then-rerank pipelines can improve ranking quality on standard retrieval benchmarks~\citep{stark,beir,nogueira-etal-2020-document}. However, their effectiveness depends on the first-stage candidate set: any answer in $A(q) \setminus R(q)$ is unavailable to the reranker, and candidate composition affects subsequent ranking~\citep{dang2013twostage,zamani2022stochastic}. The first stage must therefore balance answer coverage with a compact retrieval budget.

Dense retrieval is widely used for this stage~\citep{karpukhin2020dense,reimers-2019-sentence-bert,reimers-2020-multilingual-sentence-bert,thakur-2020-AugSBERT,reimers-2020-Curse_Dense_Retrieval,zhang2025qwen3}. In a KG, it can select the nodes with the highest embedding similarity to the query. This provides an effective baseline, but answers whose relevance depends on multi-hop relational constraints may receive low similarity scores and be excluded from later stages. Our work investigates how explicit graph structure can improve retrieval beyond this similarity-based selection.

\paragraph{Graphs as inputs to language models.}
A related line of work studies how to encode structured data so that language models can better consume graph context. \citet{talklikeagraph,letyourgraphdothetalking} show that graph-aware serialization and prompting can improve graph understanding by LLMs. These methods are complementary to our focus: they improve reasoning over graph evidence presented to the model, while we study how to retrieve that evidence from a large graph.

\paragraph{Graph-structured retrieval-augmented generation.}
Graph-based RAG methods use relational structure during retrieval and context construction~\citep{graphrag_survey}. GraphRAG builds graph-based indexes over document collections and uses community structure for query-focused summarization~\citep{graphrag}. G-Retriever retrieves task-relevant textual subgraphs for graph question answering, while GRAG combines textual and topological information from retrieved subgraphs for generation~\citep{gretriever,grag}.

HippoRAG constructs a knowledge-graph memory and uses Personalized PageRank for associative retrieval, with HippoRAG~2 further developing this memory-based approach~\citep{hipporag,hipporag2}. LightRAG combines graph-structured indexing with vector retrieval through low- and high-level retrieval, while KG$^2$RAG expands and organizes semantically retrieved seed chunks using KG relations~\citep{lightrag,kg2rag}. GFM-RAG learns a graph neural retriever over graph-indexed corpora~\citep{gfmrag}. Other methods investigate relation-path planning, path pruning, and graph-based reranking between the retriever and reader~\citep{rog,pathrag,grag_reranking}. These approaches establish graph structure as a useful component of retrieval and evidence organization.

\paragraph{Agentic graph retrieval.}
Graph-R1 and GraphRAG-R1 use reinforcement learning to optimize multi-step retrieval and reasoning over graph-structured knowledge~\citep{graph_r1,graphrag_r1}. GraphFlow learns a sequential KG retrieval policy using a flow-matching objective and outcome-level retrieval rewards~\citep{graphflow}. Our approach also learns from retrieval feedback, but uses a different computation pattern: non-learned expansion first constructs a bounded candidate graph, which a lightweight Graph Transformer encodes once to support answer selection and ranking. The learned policy selects the returned answers from this graph, with Recall@20 guiding selection, rather than controlling further graph expansion.

\paragraph{Graph encoders and learned graph selection.}
\methodname{} builds on structure-aware graph representation learning. GCN, GAT, GATv2, and GIN aggregate information from local neighborhoods~\citep{gcn,gat,gatv2,gin}, while graph transformers such as SAN, Graphormer, GraphGPS, and Exphormer study structural encoding, long-range interactions, and scalable attention~\citep{san,graphormer,graphgps,exphormer}. These models provide backbones for reasoning over retrieved subgraphs.

Neighborhood sampling and sparsification offer another route to scalable graph computation. GraphSAGE, PinSAGE, and GraphSAINT use non-learned neighborhood or subgraph sampling schemes~\citep{graphsage,pinsage,graphsaint}. Spexphormer and GRAPES introduce learned mechanisms for selecting attention connections or informative neighbors~\citep{spexphormer,grapes}. GRAPES is particularly related through its use of a learned sampling policy, including reinforcement-learning-based training, to select nodes for downstream GNN computation. Our policy instead conditions on the query and retrieved candidate graph and learns to select a compact set of answer nodes from retrieval feedback. An auxiliary ranking head learns to order those answers, supporting both candidate coverage and ranking quality.

\section{Dataset Details}
\label{app:datasets}

We evaluate \methodname{} on the three semi-structured retrieval benchmarks introduced by STARK~\citep{stark}: STARK-PRIME, STARK-MAG, and STARK-AMAZON. Each benchmark is built on top of a semi-structured knowledge base (SKB) that combines textual information associated with entities and typed relational information between entities. Given a natural-language query, the retrieval task is to identify the relevant node entities from the corresponding SKB. The queries are designed to combine textual requirements with relational constraints, making the benchmark suitable for evaluating retrieval methods that must reason over both modalities. We use the official train, validation, and test splits provided by STARK for all three datasets.

\paragraph{STARK-PRIME.}
STARK-PRIME is a biomedical retrieval benchmark built from PrimeKG \citep{primekg}, a precision-medicine knowledge graph. Its knowledge base contains ten entity types, including \texttt{disease}, \texttt{drug}, \texttt{gene/protein}, \texttt{pathway}, \texttt{anatomy}, \texttt{effect/phenotype}, \texttt{biological\_process}, \texttt{molecular\_function}, \texttt{cellular\_component}, and \texttt{exposure}. It contains eighteen relation types, covering biomedical relations such as \texttt{associated\_with}, \texttt{indication}, \texttt{contraindication}, \texttt{target}, \texttt{carrier}, \texttt{enzyme}, \texttt{transporter}, \texttt{side\_effect}, \texttt{phenotype\_present}, \texttt{phenotype\_absent}, and \texttt{expression\_present}. The textual information includes descriptions of diseases and drugs from PrimeKG, and additional textual attributes for genes/proteins and pathways from external biomedical databases. STARK-PRIME queries cover different biomedical use cases and are designed to reflect different user roles, including medical scientists, doctors, and patients. Compared with STARK-AMAZON and STARK-MAG, STARK-PRIME has fewer nodes but a denser and more diverse relational structure. The dataset contains 11,204 queries.

\paragraph{STARK-MAG.}
STARK-MAG is an academic-paper retrieval benchmark constructed from OGBN-MAG, OGBN-papers100M \citep{ogb}, and Microsoft Academic Graph \citep{microsoftgraph}. Its knowledge base contains four entity types: \texttt{paper}, \texttt{author}, \texttt{institution}, and \texttt{field\_of\_study}. The relation types include authorship, paper-field associations, citations between papers, and author-institution affiliations. Textual information is primarily associated with paper nodes and includes titles and abstracts from OGBN-papers100M, augmented with additional metadata such as venues, author names, and institution names from Microsoft Academic Graph. The queries ask for papers satisfying both textual criteria, such as topic, method, or contribution, and relational constraints, such as being authored by a particular researcher, belonging to a field, citing another paper, or being written by researchers from a given institution. STARK-MAG contains 13,323 queries.

\paragraph{STARK-AMAZON.}
STARK-AMAZON is a product-retrieval benchmark constructed from the Sports and Outdoors category of Amazon Product Reviews and Amazon Question and Answer Data. Its knowledge base contains two entity types, \texttt{product} and \texttt{brand}, and three relation types: \texttt{also\_bought}, \texttt{also\_viewed}, and \texttt{has\_brand}. The textual information is obtained from product metadata, product descriptions, prices, customer reviews, and customer Q\&A records. Brand entities are associated with their brand titles. The benchmark contains customer-oriented queries that resemble real product-search requests. These queries often describe desired product properties, such as functionality, quality, style, or use case, while also including relational constraints such as brand membership or relationships to complementary and substitute products. STARK-AMAZON contains 9,100 queries, with a large fraction of queries having multiple correct answers.

\begin{table}[b]
\centering
\caption{Statistics of the semi-structured knowledge bases used in STARK. The statistics are reported by the original STARK paper.}
\label{tab:skb_statistics}
\scalebox{0.85}{
\begin{tabular}{lrrrrrr}
\toprule
Dataset & 
\# Node Types & 
\# Relation Types & 
Avg. Degree & 
\# Nodes & 
\# Edges & 
\# Tokens \\
\midrule
STARK-PRIME  & 10 & 18 & 62.6 & 129,375   & 8,100,498  & 31,844,769 \\
STARK-MAG    & 4  & 4  & 10.6 & 1,872,968 & 19,919,698 & 212,602,571 \\
STARK-AMAZON & 2  & 3  & 3.0  & 1,032,407 & 3,886,603  & 592,067,882 \\
\bottomrule
\end{tabular}
}
\end{table}

\begin{table}[t]
\centering
\caption{Statistics of the STARK retrieval datasets. We use the official train/validation/test splits.}
\label{tab:stark_query_statistics}
\scalebox{0.9}{
\begin{tabular}{lrrrr}
\toprule
Dataset & 
\# Queries & 
\# Queries w/ Multiple Answers & 
Avg. \# Answers & 
Train / Val / Test \\
\midrule
STARK-PRIME  & 11,204 & 4,188 & 2.56  & 0.55 / 0.20 / 0.25 \\
STARK-MAG    & 13,323 & 6,872 & 2.78  & 0.60 / 0.20 / 0.20 \\
STARK-AMAZON & 9,100  & 7,082 & 17.99 & 0.65 / 0.17 / 0.18 \\
\bottomrule
\end{tabular}
}
\end{table}

\section{A hard problem for dense retrieval}
\label{app:dense-lb}

We formalize a family of KGs where multi-hop reachability induces a hard classification subproblem. In this setting, answering global reachability queries based only on a single node embedding requires each node to store a massive amount of information. With access to the graph at query time, query-conditioned message passing solves the same problem using compact binary node states.

\paragraph{Dense retrieval model class.} We define a dense retrieval model class as any model that decides whether a node answers a query based exclusively on a query embedding and a node embedding. While this is typically implemented via cosine similarity, we generalize this to any retrieval function that takes both embeddings as input and returns a boolean match.

We treat embeddings as finite sequences of binary bits, rather than infinite-precision real numbers.
This is typical for information theoretic and communication complexity results;
while it is not a perfect match for practice
(where it is very difficult to gain any advantage from using 17 rather than 32 bits of information),
it is far more theoretically reasonable than infinite precision.
For instance,
in machines endowed only with infinite-precision arithmetic
and a floor operation,
not only NP but even PSPACE are decidable in randomized polynomial time \citep{ram-power}.

\begin{definition}[Relation tracing graphs] \label{def:rel-tracing}
Let $G$ be a finite directed graph $G=(V,E)$,
with nodes $V = [n]$.
Define the edges by $k$ permutations $\sigma_1, \dots, \sigma_k$ on $V$: a directed edge $u \xrightarrow{r} v$ exists if and only if $v = \sigma_r(u)$ for some relation type $r \in \{1, \dots, k\}$. 
Each node in $V$ has both in-degree and out-degree of $k$.
Let $r = (r_1, \dots, r_\ell)$ be a sequence of $\ell$ relation types,
and define 
$\sigma_r = \sigma_{r_\ell} \circ \sigma_{r_{\ell-1}} \circ \cdots \circ \sigma_{r_1}$.
Fix $L \ge 1$
and define $\mathcal R := \bigcup_{\ell = 0}^{L} [k]^\ell$ to be the space of possible queries.
A query in this model asks,
given an $s \in V$ and $r \in \mathcal R$,
to find $t = \sigma_r(s) \in V$.
Also define $\sigma^{-1}_r$ as the inverse mapping,
$s = \sigma^{-1}_r(t)$.
\end{definition}

\subsection{Hardness for independent decision-making}
We begin with a self-contained result, with explicit constants,
showing that per-node dense retrieval requires extremely large node embeddings for relation tracing graphs.

\begin{theorem} \label{thm:hard-decisions}
Sample a relation-tracing graph $G$ with independent uniform permutations
 \(\sigma_1,\dots,\sigma_k\) of \([n]\),
 where $k \ge 2$.
After observing $G$,
we will assign binary strings $x_t \in \{0, 1\}^*$
to the vertices $t \in [n]$,
to be used as inputs to a decision function
$
    D:\{0,1\}^*\times[n]\times \mathcal R \to \{0,1\}
$,
which is fixed independently of $G$.\footnote{Otherwise, $x_t$ could simply encode the integer $t$ in length $\log n$, and $D$ could answer by lookup on the known graph.}
Suppose that the features are assigned such that, for every
\(t,s\in[n]\) and every $r \in \mathcal R$,
\[
    D(x_t,s,r)=1
    \qquad\text{if and only if}\qquad
    t=\sigma_r(s).
\]
If $k^L \le \frac{k-1}{k+4} n$,
then
using $\abs{x_t}$ to denote the length of the binary string $x_t$,
we must have
\[
    \E_G\left[\frac1n\sum_{t=1}^n \abs{x_t} \right]
    >
    \frac14 k^L \log_2\left( \frac n k \right) - \frac12
.\]
In particular, choosing $k^L = \Theta(n / \log n)$
yields that
$
    \E_G\left[\frac1n\sum_{t=1}^n \abs{x_t} \right]
    = \Omega(n)
$.
\end{theorem}
\begin{proof}
For $t \in [n]$ and $r \in \mathcal R$,
define $Z_t(r)=\sigma_r^{-1}(t)$.
Correctness of \(D\) implies that \(x_t\) determines the whole function
\(Z_t\):  given \(x_t\) and \(r\), we can recover \(Z_t(r)\) by
searching over \(s\in[n]\) for the unique value satisfying
$D(x_t,s,r)=1$.
Therefore, if \(T\) is independently uniform on \([n]\) and \(X=x_T\), 
then \(Z_T\) is a deterministic function of \(X\).
Hence the Shannon entropy $\ent$ satisfies
\[
    \ent(Z_T) \le \ent(X).
\]
Since \(T\) is uniform and independent of \(G\),
\[
    \ent(Z_T)\ge \ent(Z_T\mid T)
    =
    \frac1n\sum_{t=1}^n \ent(Z_t).
\]
By symmetry, this equals \(\ent(Z_t)\) for any fixed choice of \(t\).
The general result on binary strings \cref{thm:binary-string-entropy}
then implies that
\[
    \E \abs{X}
    \ge \frac12 \ent(X) - \frac12
    \ge \frac12 \ent(Z_t) - \frac12
.\]
The result follows from \cref{thm:permutation-table-entropy}.
\end{proof}

\begin{lemma} \label{thm:binary-string-entropy}
    Let $X$ be a random binary string with random length $\abs{X}$.
    Then the Shannon entropy, measured in bits, satisfies
    \[ \ent(X) \le 2 \E \abs{X} + 1 .\]
\end{lemma}
\begin{proof}
    One way to see this is to construct a prefix-free code for $X$:
    we first encode $\abs X$
    in unary, $\abs X$ copies of $1$ followed by a $0$,
    followed by $X$.
    This has length $2 \abs X + 1$,
    and the result follows by the Shannon source-coding theorem.
\end{proof}

\begin{lemma}[Entropy lower bound] \label{thm:permutation-table-entropy}
Let \(k\ge 2\), and let
\(\sigma_1,\dots,\sigma_k\) be independent uniform random permutations
of \([n]\).
Fix \(t\in[n]\).
Let $Z_t \in [n]^{\mathcal R}$ denote the random function
$Z_t(r) = \sigma_r^{-1}(t)$
for all queries up to length $L$, $r \in \mathcal R$.
Let $\eta_k = \frac{k-1}{k+4}$.
If $k^L \le \eta_k n$,
then the Shannon entropy of $Z_t$, measured in bits, is lower-bounded as
\[ \ent(Z_t) > \frac12 k^L \log_2\left(\frac n k\right) .\]

\end{lemma}
\begin{proof}
We will consider iteratively expanding each relation,
keeping track of the answers to all possible queries at a given depth through
the sets
\[
    S_0=\{t\},
    \qquad
    S_{\ell+1}=\bigcup_{j=1}^k \{ \sigma_j^{-1}(s) : s \in S_\ell \}
.\]

In our analysis,
we will track when we first observe each value of permutation,
$\sigma_j^{-1}(s)$,
and see how much entropy is ``revealed'' in that decision.
To do this,
it will be helpful to
track the set of vertices first observed at depth $\ell$; we call this set $A_\ell$.
That is,
let
$B_\ell = \bigcup_{h=0}^\ell S_h$
be the set of all vertices ever seen up to depth \(\ell\), with $B_{-1} = \{ \}$;
then define $A_\ell = S_\ell \setminus B_{\ell-1}$.

We first establish a useful fact:
by simply counting the number of all possible tuples
of length up to $\ell$, we have
\[
    \abs{B_\ell}
    \le 1+k+\cdots+k^\ell
    = \frac{k^{\ell+1}-1}{k-1}
.\]
For $\ell < L$,
since $k^L \le \eta_k n$
we further have
\[
    \abs{B_\ell}
    \le \frac{k^{\ell+1}-1}{k-1}
      < \frac{k^L}{k-1}
    \le \frac{\eta_k}{k-1} n
    = \frac{1}{k+4} n
.\]

Now, consider the expansion from a known $A_\ell, B_\ell$
to depth $\ell+1$.
Any vertices we saw before step $\ell$ (those in $B_{\ell-1}$)
have already been fully expanded in the previous step;
our only new potential sources of randomness
come from vertices in $A_\ell$.
For each newly observed vertex $a \in A_\ell$
and each relation $j \in [k]$,
we will observe the $k \abs{A_i}$ values
$\sigma_j^{-1}(a)$,
none of which have previously been revealed.

For each $j$,
these values will be jointly sampled without replacement from
the not-yet-used values of $\sigma_j^{-1}$.
While we have not kept track of exactly which points those are,
we know anything outside of $B_\ell$ cannot have been used yet.
Since $\abs{B_\ell} \le \frac{1}{k+4} n$,
we are choosing $\abs{A_\ell}$ values out of a set of size at least $\left( 1 - \frac{1}{k+4} \right) n = \frac{k+3}{k+4} n$.
As these are ordered decisions,
this observation has conditional entropy
\[
    \log_2\left( \left(\text{\# unused values for $\sigma_j$}\right)_{(\abs{A_\ell})} \right)
    \ge 
    \log_2\left( \left( \frac{k+3}{k+4} n - \abs{A_\ell} + 1 \right)^{\abs{A_\ell}} \right)
,\]
using the Pochammer symbol $a_{(b)} = a (a - 1) \cdots (a - b + 1)$.
As $\abs{A_\ell} \le k^\ell \le k^L \le \eta_k n$,
we have that
$\frac{k+3}{k+4} n - \abs{A_\ell} + 1 > \left(\frac{k+3}{k+4} - \frac{k-1}{k+4}\right) n = \frac{4}{k+4} n \ge \frac{n}{k}$ for $k \ge 2$.

The $k$ permutations are independent,
so the total entropy added in the $\ell$th expansion is $k$ times this.

As $\abs{A_\ell}$ is random,
the exact amount of entropy added at each step is random.
By iterating conditional expectations,
however,
the overall entropy is the sum of 
the \emph{expected} conditional entropy at each step.
Thus
\[
    \ent(Z_t)
    \ge \sum_{\ell=0}^{L-1} k \E \abs{A_\ell} \log_2 \frac{n}{k}
    = \log_2\left( \frac n k \right) k \sum_{\ell=0}^{L-1} \E \abs{A_\ell}
.\]
We will show
$\E \abs{A_\ell} \ge \frac12 k^\ell$, and so
\[
    \ent(Z_t)
    \ge \log_2\left(\frac n k\right) k \sum_{\ell=0}^{L-1} \frac12 k^{\ell}
    = \frac12 \log_2\left(\frac n k \right) \; k \frac{k^L - 1}{k - 1}
    > \frac12 k^L \log_2\left( \frac n k \right)
,\]
as desired.
It thus only remains to show this lower bound on $\E \abs{A_\ell}$.

In our expansion of level $\ell+1$,
some of the newly sampled values of $\sigma_j(a)$
may be a node we've already seen (one in $B_\ell$),
while others may be outside $B_\ell$ but also selected by a different $\sigma_{j'}(a')$.
There are $\abs{B_\ell}$ nodes in the first set.
The exact number in the second set is complicated,
but if we expand the relations in order,
when expanding $\sigma_j$
we will have previously expanded $(j-1) \abs{A_\ell}$ candidates,
and so 
there cannot be more than $(j-1) \abs{A_\ell}$ of them.

We are selecting from the candidate pool of all points not yet used by relation $j$,
of which there will be at least $n - \abs{B_\ell} \ge \frac{k+3}{k+4} n$ candidates.
Thus, in expanding the $\abs{A_\ell}$ candidates for relation $j$,
the expected number of ``bad'' choices is at most
\begin{align*}
       \abs{A_\ell} \frac{\abs{B_\ell} + (j-1) \abs{A_\ell}}{n - \abs{B_\ell}}
  &\le \frac{k+4}{k+3} \frac1n \abs{A_\ell} (\abs{B_\ell} + (j-1) \abs{A_\ell})
\\&\le \frac{k+4}{k+3} \frac1n \abs{A_\ell} \left( \frac{k^{\ell+1}}{k-1} + (j-1) k^\ell \right)
\\&\le \frac{k+4}{k+3} \frac1n \abs{A_\ell} (j+1) k^\ell
  &&\text{since $k/(k-1) \le 2$ for $k \ge 2$}
\\&\le \frac{k+4}{k+3} \frac{k-1}{k+4} \abs{A_\ell} (j+1) k^{\ell - L}
  &&\text{using $k^L \le \eta_k n$ so $\frac1n \le \eta_k k^{-L}$}
.\end{align*}
Since $\sum_{j=1}^k (j+1) = \frac12 k (k+3)$,
the expected total number of ``bad'' choices in the full expansion across all relations is at most
\[
\abs{A_\ell} \frac12 (k-1) k k^{\ell - L}
,\]
and so
\[
    \E \abs{A_{\ell+1}}
    \ge k \E \abs{A_{\ell}} \left( 1 -
         \frac12 (k-1) k^{\ell - L}
    \right)
.\]
We also know that $\abs{A_0} = 1$.
Thus, unrolling the recursion
and applying $\prod_i (1 - a_i) \ge 1 - \sum_i a_i$
for $a_i \in [0, 1]$,
we have
\begin{align*}
       \E \abs{A_\ell}
   &\ge k^\ell \prod_{i=0}^{\ell-1} \left( 1 - \frac12 (k-1) k^{i - \ell} \right)
   \ge k^\ell - \frac12 (k-1) \sum_{i=0}^{\ell-1} k^{i} \\
     &= k^\ell - \frac12 (k-1) \frac{k^\ell - 1}{k-1}
     = \frac12 (k^\ell + 1)
     > \frac12 k^\ell.\qedhere
\end{align*}
\end{proof}

\subsection{Hardness of retrieval by scoring}

We can also obtain a lower bound in the more common case
when retrieval returns the $B$ highest-scoring nodes, for any fixed $B$,
rather than requiring an independent decision for each possible answer.
Here, correctness only constrains the relative scores of candidates,
so we need to consider blocks of $B+1$ embeddings.
This class of algorithms subsumes those of \cref{thm:hard-decisions},
but the result requires queries of logarithmic length.

\begin{theorem} \label{thm:hard-argmax}
Fix an integer $B \ge 1$.
Sample a relation-tracing graph $G$ as in \cref{def:rel-tracing},
with independent uniform permutations $\sigma_1, \dots, \sigma_k$ of $[n]$,
where $k \ge 3$ and $n > B$.
After observing $G$,
assign binary strings $x_t \in \{0,1\}^*$ to the vertices $t \in [n]$,
to be used as inputs to a scoring function
\[
    F : \{0,1\}^* \times [n] \times \mathcal R \to \R
\]
which is fixed independently of $G$.
Let $R_B(s,r)$ consist of the $B$ vertices with largest values of $F(x_t,s,r)$,
with ties resolved by an ordering of vertex identifiers fixed independently of $G$.
Suppose that the features are assigned such that, for every
$s \in [n]$ and $r \in \mathcal R$,
\[
    \sigma_r(s) \in R_B(s,r)
,
\]
i.e.\ the correct answer is always among the retrieved vertices.
There is a constant $C_B > 0$, depending only on $B$, such that,
if $L \ge \ceil{C_B \log_2 n} + 1$,
then
as $n \to \infty$,
\[
    \E_G\left[\frac1n\sum_{t=1}^n \abs{x_t}\right]
    \ge \left(\frac{1}{2(B+1)}-o(1)\right)
        \bigl(\log_2(n!) - n\log_2 B\bigr)
    = \Omega\Bigl( \frac n B \log \frac n B \Bigr)
.\]
\end{theorem}
\begin{proof}
Consider the directed graph on ordered $(B+1)$-tuples of distinct vertices of $[n]$,
with edges
$(u_1,\dots,u_{B+1}) \to (\sigma_i(u_1),\dots,\sigma_i(u_{B+1}))$,
for $i \in \{2,3\}$.
\citet[Theorem 2.2]{joel:quickly-transitive} show that,
with probability $1-o(1)$ over $\sigma_2,\sigma_3$,
this graph expands in both directions:
for a constant $c_B > 0$ depending only on $B$,
every set $S$ containing at most half the tuples
has at least $c_B\abs S$ out-neighbors outside $S$
and at least $c_B\abs S$ in-neighbors outside $S$.
Starting from any two tuples,
the forward ball from the first and the backward ball from the second
therefore each grow to contain more than half of the $n!/(n-B-1)!$ tuples
in $\bigO_B(\log n)$ steps.
Their intersection gives a directed path between the two tuples.

Consequently, there is a constant $C_B$ depending only on $B$ such that
with probability $1-o(1)$ over $\sigma_2,\sigma_3$,
every ordered $(B+1)$-tuple of distinct vertices can be mapped to every other
by a word over $\{2,3\}$ of length at most $\ceil{C_B\log_2 n}$.
Call this event $\mathcal E$.
Fix any $\sigma_2,\dots,\sigma_k$ for which $\mathcal E$ holds;
note that this event is independent of $\sigma_j$ for $j > 3$.
The remainder of the proof
deals with randomness over the remaining independent uniform permutation $\sigma_1$.

Let $m = \floor{n/(B+1)}$,
and consider the disjoint blocks
$P_j = \{(B+1)(j-1)+1,\dots,(B+1)j\}$ for $j \in [m]$.
Write $Y_j = (x_t)_{t \in P_j}$ for the tuple of strings in a block.
For a fixed value $y$ of this tuple,
let $\mathcal P_j(y)$ be the set of permutations $\pi$
that can be used as $\sigma_1$ in a graph admitting a correct feature assignment
with $Y_j=y$.
We claim that, for each $s \in [n]$,
there are at most $B$ possible values of $\pi(s)$ among $\pi \in \mathcal P_j(y)$.

To see this, assume for the sake of contradiction that
$\pi_1,\dots,\pi_{B+1} \in \mathcal P_j(y)$
have distinct values at some $s$.
List the vertices of $P_j$ as $t_1,\dots,t_{B+1}$.
By $\mathcal E$, there is a word $p$ over $\{2,3\}$,
of length at most $\ceil{C_B\log_2 n}$,
such that
\[
    \sigma_p(\pi_h(s)) = t_h
    \qquad\text{for every } h \in [B+1]
.
\]
The query $r=(1,p)$ has length at most $L$;
by \cref{def:rel-tracing}, its correct answer is $t_h$
in the graph with $\sigma_1=\pi_h$.
The strings in $P_j$ are the same in all these feature assignments,
so the relative order of their scores, including tie-breaking, is also the same.
The lowest-ranked vertex in $P_j$ has $B$ vertices ahead of it,
and therefore cannot belong to $R_B(s,r)$ in any of these graphs.
But it is the correct answer in one of them, a contradiction.

Thus each value of $Y_j$ leaves at most $B$ possibilities for each of
$\sigma_1(1),\dots,\sigma_1(n)$,
and hence at most $B^n$ possibilities for $\sigma_1$.
Its conditional Shannon entropy, measured in bits, therefore satisfies
$\ent(\sigma_1 \mid Y_j) \le n\log_2 B$.
Since $\sigma_1$ is uniform, we obtain
\[
    \log_2(n!) - n\log_2 B
    \le \ent(\sigma_1) - \ent(\sigma_1 \mid Y_j)
    \le \ent(Y_j)
    \le \sum_{t \in P_j} \ent(x_t)
    \le 2\E\left[\sum_{t \in P_j}\abs{x_t}\right] + B+1
,\]
where the last inequality is \cref{thm:binary-string-entropy}.
Summing over the $m$ disjoint blocks,
and then averaging over the remaining permutations conditional on $\mathcal E$,
gives
\[
    \E_G\left[ \frac1n\sum_{t=1}^n \abs{x_t} \,\middle|\, \mathcal E \right]
    \ge \frac{m}{2n}\bigl(\log_2(n!) - n\log_2 B - (B+1)\bigr)
.\]
Lengths are nonnegative,
and $\Pr(\mathcal E)=1-o(1)$,
so
\begin{align*}
    \E_G\left[\frac1n\sum_{t=1}^n\abs{x_t}\right]
    &\ge \Pr(\mathcal E)\frac{m}{2n}
        \bigl(\log_2(n!) - n\log_2 B - (B+1)\bigr) \\
    &= \left(\frac{1}{2(B+1)}-o(1)\right)
        \bigl(\log_2(n!) - n\log_2 B\bigr)
.
\end{align*}
The result follows since, by Stirling's approximation,
$\log_2(n!) = \Theta(n\log n)$.
\end{proof}

Taking $B=1$ gives a lower bound of
$\left(\frac14-o(1)\right)\log_2(n!)$ for retrieval by argmax.
For any fixed $B \ge 1$, $k \ge 3$,
and $L = \Theta(\log n)$ with a sufficiently large constant,
\cref{thm:hard-argmax} therefore requires $\Omega(n\log n)$ bits per node on average.

\subsection{Query-conditioned message passing on a fixed candidate graph}

We now allow the node representations to depend on the query through message passing
on a supplied candidate graph.
Our construction will allow us to reduce the node states from $\Omega(n \log n)$ bits per node
(as in \cref{thm:hard-argmax})
to just $\bigO(\log n)$.

We isolate two operations used by the encoder:
query-conditioned attention over typed edges, and a pointwise update of node states.
The query specifies a relation sequence, and a binary marker identifies the source.
The following construction scores all candidates in one forward pass;
its layers propagate information without changing the candidate set.

\begin{theorem} \label{thm:iterative-tracing}
    Fix integers $k,L \ge 1$.
    There is an $L$-layer query-conditioned message-passing network
    with hidden width $\bigO(Lk)$,
    softmax attention with finite logits, and pointwise ReLU updates,
    whose parameters depend only on $k$ and $L$, with the following property.
    Take any relation-tracing graph $G$ as in \cref{def:rel-tracing},
    a query $(s,r)$ with $\abs r=\ell \le L$,
    and any fixed candidate subgraph $\widetilde G=(\widetilde V,\widetilde E)$ of $G$
    containing $s$.
    Encode $r$ by $L$ one-hot vectors over $\{0,\dots,k\}$,
    using the extra symbol $0$ to pad positions after $\ell$.
    Give each node this query encoding and the single bit $\indic[v=s]$,
    and each edge its relation type and the query encoding.
    The network assigns score $1$ to precisely the vertices reachable from $s$
    along a walk in $\widetilde G$ with relation sequence $r$,
    and score $0$ to every other vertex.
    In particular, if $\widetilde G$ contains the entire tracing path,
    the unique highest-scoring vertex is $\sigma_r(s)$.
    No node identifiers or graph-dependent parameters are required.
\end{theorem}
\begin{proof}
    Write $r_i=0$ for $i>\ell$, and let $g_i(j)=\indic[j=r_i]$.
    This gate is a function of the query encoding and the edge type:
    for one-hot vectors $z_i,e_j$, it can be computed as
    $\sum_{h=0}^k (z_{i,h}+e_{j,h}-1)_+$,
    where $(x)_+=\max\{x,0\}$ is the ReLU function.
    At each vertex, add a self-message of type $0$,
    and write $\mathcal N_v$ for the resulting set of incoming pairs $(u,j)$.
    Each $\mathcal N_v$ is nonempty, has size at most $k+1$,
    and contains at most one pair of each type.

    Initialize $a_v^{(0)}=\indic[v=s]$ and set $\tau=\ln(16(k+1))$.
    At layer $i$, use attention logits $\tau g_i(j)$ and scalar values
    $a_u^{(i-1)}+g_i(j)$:
    \[
        \alpha_{v,u,j}^{(i)}
        = \frac{\exp(\tau g_i(j))}
        {\sum_{(w,h)\in\mathcal N_v}\exp(\tau g_i(h))},
        \qquad
        b_v^{(i)}
        = \sum_{(u,j)\in\mathcal N_v}
            \alpha_{v,u,j}^{(i)}\bigl(a_u^{(i-1)}+g_i(j)\bigr)-1
    .\]
    The pointwise update is
    \[
        a_v^{(i)}=\rho\bigl(b_v^{(i)}\bigr),
        \qquad
        \rho(b)=2\left(b-\frac14\right)_+
                   -2\left(b-\frac34\right)_+
    .\]
    Thus $\rho(b)=0$ for $b\le 1/4$ and $\rho(b)=1$ for $b\ge 3/4$.
    The query encoding is carried unchanged through the layers.

    Suppose the preceding states are binary.
    If there is an incoming pair $(u,r_i)$ with $a_u^{(i-1)}=1$,
    its attention weight is at least
    \[
        \frac{e^\tau}{e^\tau+k}
        \ge 1-ke^{-\tau}
        > \frac{15}{16}
    .\]
    Its value is $2$, and all other values are nonnegative,
    so $b_v^{(i)}>7/8$ and $a_v^{(i)}=1$.
    Otherwise every incoming value is at most $1$,
    so $b_v^{(i)}\le 0$ and $a_v^{(i)}=0$.
    In particular, a missing edge of the requested type produces a zero state,
    even when a different incoming edge comes from an active vertex.
    The query-conditioned term in the value is what ensures this property.

    Consequently, for $i\le\ell$ the update is exactly
    \[
        a_v^{(i)}=1
        \quad\text{exactly when}\quad
        \text{there is an edge }u\xrightarrow{r_i}v\text{ in }\widetilde G
        \text{ with }a_u^{(i-1)}=1
    ,\]
    while for $i>\ell$ the type-$0$ self-message preserves the state.
    Therefore, the active vertices are precisely the endpoints of
    walks from $s$ with relation sequence $(r_1,\dots,r_{\min(i,\ell)})$.
    In a relation-tracing graph there is at most one such endpoint.
    The final scores $a_v^{(L)}$ therefore have the claimed property.

    Each node carries $L(k+1)$ query bits and one activity bit;
    the gate and the pointwise update also have width $\bigO(Lk)$.
    All parameters used above depend only on $k,L$.
\end{proof}

The construction uses finite logits and has a margin:
perturbing each $b_v^{(i)}$ by at most $1/8$ before the pointwise update
leaves every $a_v^{(i)}$ unchanged.
Indeed, the inactive case remains below $1/4$,
and the active case remains above $3/4$.
The update restores a binary state at every layer, so these perturbations do not accumulate.

This result isolates the query-conditioned typed attention and pointwise updates of the encoder.
While we do not model the normalization layers or other implementation details used in our practical models,
it gives an expressivity statement that there exists a narrow network that can trace a given relation
-- as long as the initial \textsc{K-hop-with-filtering} subgraph extraction
retains the full path of that relation.
As usual for such results,
it does not guarantee that practical training will recover this model,
or any successful model,
but it provides an existence proof that allowing query-conditioned message passing
in a realistic architecture
escapes the negative results of \cref{thm:hard-decisions,thm:hard-argmax}.

\section{Selection from a fixed candidate graph}
\label{sec:appendix:selection}

Fix a query $q$ with a nonempty answer set $A(q)$.
Let $C_q$ be the set of valid output candidates in the graph returned by
\textsc{K-hop-with-filtering}; the graph may also contain intermediate nodes
that the encoder uses but that are not eligible outputs.
The candidate graph is fixed before learned selection begins.
Fix an integer output budget $B \ge 1$ and assume throughout that $\abs{C_q}\ge B$.
Let $S \subseteq C_q$ contain $B$ selected nodes.
There is no connectivity constraint on $S$:
intermediate nodes can contribute to the encoder's computation without occupying output slots.
Define
\[
    R_q(S)=\frac{\abs{A(q)\cap S}}{\abs{A(q)}}
.\]

\subsection{Candidate coverage and selection error}

\begin{proposition}[Oracle decomposition] \label{prop:selection-oracle}
    The best recall attainable by selecting from $C_q$ is
    \[
        O_B(q):=\max_{S\subseteq C_q:\,\abs S=B} R_q(S)
        =\frac{\min\{B,\abs{A(q)\cap C_q}\}}{\abs{A(q)}}
    .\]
    For every such $S$, the recall deficit decomposes as
    \[
        1-R_q(S)
        =\underbrace{\frac{\abs{A(q)\setminus C_q}}{\abs{A(q)}}}_{\text{missing candidates}}
        +\underbrace{\frac{(\abs{A(q)\cap C_q}-B)_+}{\abs{A(q)}}}_{\text{output budget}}
        +\underbrace{\bigl(O_B(q)-R_q(S)\bigr)}_{\text{selection error}}
    .\]
\end{proposition}
\begin{proof}
    A set of at most $B$ candidates can contain at most $\min\{B,\abs{A(q)\cap C_q}\}$ answers.
    This is achieved by selecting that many answer nodes and filling any remaining slots
    with other candidates.
    The decomposition follows from
    $\abs{A(q)\cap C_q}=(\abs{A(q)\cap C_q}-B)_++\min\{B,\abs{A(q)\cap C_q}\}$.
\end{proof}

The three terms remain separate after averaging over queries.
The first term measures the quality of the first-stage subgraph extraction,
such as \textsc{K-hop-with-filtering}.
The second measures any unavoidable additional loss due to our limited output selection budget $B$.
The third and final term gives the quality of the selection algorithm:
it cannot recover an answer excluded from $C_q$,
but the selection error measures how many of the possible candidates it selects.
This can also be affected by the choice of $C_q$:
for instance, in the tracing construction of \cref{thm:iterative-tracing},
the network might fail to identify an answer which is actually present if some piece of the relation path is missing.

\subsection{Ranking guarantees}

We first describe the best selection rule given the model's input.
Let $x$ include the query representation, candidate graph, and node and edge features,
so that $C=C_q$ is determined by $x$.
Treat the answer set $A$ as random conditional on $x$, and define
\[
    w_v(x)=\E\left[\frac{\indic[v\in A]}{\abs A}\,\middle|\,x\right]
    \quad \text{for } v\in C
.\]
The normalization by $\abs A$ reflects the benchmark's average of per-query recalls.

\begin{proposition}[Recall-optimal scores and ranking error] \label{prop:selection-ranking}
    Let a policy select $B$ candidates using only $x$,
    with randomness which is independent of $A$ given $x$.
    Its conditional expected recall satisfies
    \[
        \E[R_q(S)\mid x]
        =\sum_{v\in C}\Pr(v\in S\mid x)w_v(x)
        \le \sum_{v\in S^\star}w_v(x)
    ,\]
    where $S^\star$ consists of the $B$ candidates with largest $w_v(x)$.
    Equality is attained when the policy always selects $S^\star$.
    If $\widehat S$ instead consists of the $B$ largest predicted scores $\widehat w_v(x)$,
    with $\abs{\widehat w_v(x)-w_v(x)}\le\eps$ for every $v\in C$,
    then
    \[
        0\le\sum_{v\in S^\star}w_v(x)-\sum_{v\in\widehat S}w_v(x)
        \le 2B\eps
    .\]
\end{proposition}
\begin{proof}
    Conditional independence and linearity of expectation give the first equality.
    Each set of $B$ nodes has weight at most that of $S^\star$;
    averaging over a randomized choice cannot increase this maximum.

    For the second claim,
    notice that $\abs{S^\star \setminus \widehat S} = \abs{\widehat S \setminus S^\star} \le B$,
    and that by the definition of $\widehat S$, we know that
    $\sum_{v \in S^\star \setminus \widehat S} \widehat w_v \le \sum_{v \in \widehat S \setminus S^*} \widehat w_v$.
    Hence
    \begin{align*}
           \sum_{v\in S^\star} w_v - \sum_{v\in\widehat S} w_v
      &  = \sum_{v \in S^\star \setminus \widehat S} w_v
         - \sum_{v \in \widehat S \setminus S^\star} w_v
    \\&\le \sum_{v \in S^\star \setminus \widehat S} (w_v - \widehat w_v)
         + \sum_{v \in \widehat S \setminus S^\star} (\widehat w_v - w_v)
    \\&\le \sum_{v \in S^\star \setminus \widehat S} \abs{w_v - \widehat w_v}
         + \sum_{v \in \widehat S \setminus S^\star} \abs{\widehat w_v - w_v}
       \le 2 B \eps
    .\qedhere\end{align*}
\end{proof}

This optimum is over all policies using $x$.
It can be smaller than $\E[O_B(q)\mid x]$, since the oracle observes the answer labels;
a particular model class may impose further restrictions.
This $w_v$ is also the Bayes-optimal predictor
for squared-loss regression on the targets $\indic[v\in A]/\abs A$.

A pairwise ranking objective also controls selection error without requiring calibrated scores.
The next result applies to the BPR logistic loss \citep{bpr};
it is an elementary adjustment of known results \citep[Theorem 1]{Shi24:Lower-LeftPartialAUC}.

\begin{proposition} \label{prop:selection-bpr}
    Fix $q$, and let $S_f$ be the $B$ highest-scoring candidates under scores $f_v$,
    resolving ties by a fixed ordering.
    Define the unnormalized loss over all positive--negative candidate pairs by
    \[
        \mathcal L_{\mathrm{pair}}(f;q)
        =\sum_{u\in A(q)\cap C_q}\;\sum_{v\in C_q\setminus A(q)}
            \ln\bigl(1+\exp(f_v-f_u)\bigr)
    .\]
    Then
    \[
        0\le O_B(q)-R_q(S_f)
        \le \frac{1}{\abs{A(q)}}
            \sqrt{\frac{\mathcal L_{\mathrm{pair}}(f;q)}{\ln 2}}
    .\]
\end{proposition}
\begin{proof}
    Write $h = \abs{S_f\cap A(q)}$
    for the number of correctly selected positive answers.
    There are $\abs{A(q) \cap C_q} - h$ missed positives,
    and $B - h$ selected negatives.
    Each selected negative must have score at least as high as that of each missed positive;
    each of these $(\abs{A(q) \cap C_q} - h)(B - h)$ pairs
    contributes at least $\ln 2$ to the overall loss.
    Letting
    $d = \min\{B, \abs{A(q) \cap C_q} \} - h = \abs{A(q)} \bigl( O_B(q) - R_q(S_f) \bigr)$,
    notice that $\abs{A(q) \cap C_q} - h \ge d$
    and also
    $B - h \ge d$;
    thus
    $\mathcal L_{\mathrm{pair}}(f;q) \ge d^2\ln 2$.
\end{proof}

Combining \cref{prop:selection-oracle,prop:selection-bpr} bounds the recall loss by
candidate omissions, the output budget, and a pairwise ranking term.
The bound concerns the sum over all candidate pairs.
For a mean pairwise loss, it includes the corresponding factor $\abs{A(q)\cap C_q}\abs{C_q\setminus A(q)}$
inside the square root.
A small loss on pairs from one sampled subset alone gives no bound on unobserved pairs.
If selection and ranking use different heads, the guarantee applies to the scores
that choose the output set, not automatically to an auxiliary ranking head.
This demonstrates that a model with low BPR loss
can indeed achieve low selection error.

\subsection{Learning from feedback on proposed nodes}

Our motivation for RL is the feedback available during training.
Benchmark datasets provide full answer sets, but in applications it can be much more feasible
to judge a small set of proposed nodes than to enumerate all relevant nodes in the graph.
With a fixed candidate graph, this is a contextual bandit problem:
the model proposes a set and receives feedback on that proposal.
Per-node judgments give feedback on the individual selected items as well.
Supervised methods can also learn from such observations;
unequal observation probabilities matter when learning from logged feedback
\citep{swaminathan:bandit-feedback}.

To make the required feedback precise, fix $q$ and let
$U_q(S)=\sum_{v\in S}u_q(v)$ be an additive utility with bounded, fixed node values.
For benchmark recall, $u_q(v)=\indic[v\in A(q)]/\abs{A(q)}$;
for the number of relevant nodes retrieved, $u_q(v)=\indic[v\in A(q)]$.
Precision uses $u_q(v)=\indic[v\in A(q)]/B$.
A policy samples an ordered list $\sigma$ of $B$ distinct candidates,
with entries $v_i$ and underlying set $S(\sigma)$.
For example, a Plackett--Luce policy with logits $\ell_\theta(x,v)$ has probability
\[
    \pi_\theta(\sigma\mid x)
    =\prod_{i=1}^{B}
        \frac{\exp(\ell_\theta(x,v_i))}
        {\sum_{v\in C_q\setminus\{v_1,\dots,v_{i-1}\}}
            \exp(\ell_\theta(x,v))}
.\]
The order specifies sampling without replacement; the utility depends only on the set.

\begin{proposition}[Policy gradients from proposed-node feedback] \label{prop:selection-feedback}
    Suppose $\pi_\theta$ is differentiable and positive on every such list.
    For each $v\in S(\sigma)$, let $Z_v$ indicate whether its utility is observed,
    with known probability $p_v(\sigma)=\Pr(Z_v=1\mid q,\sigma)>0$.
    Set
    \[
        \widehat U(\sigma)
        =\sum_{v\in S(\sigma)}\frac{Z_v u_q(v)}{p_v(\sigma)},
        \qquad
        J_q(\theta)=\E_{\sigma\sim\pi_\theta}[U_q(S(\sigma))]
    .\]
    Then $\E[\widehat U\mid q,\sigma]=U_q(S(\sigma))$, and for any baseline $c_q$
    independent of the sampled list and its feedback,
    \[
        \E\left[(\widehat U-c_q)\nabla_\theta\log\pi_\theta(\sigma\mid x)
            \,\middle|\,q\right]
        =\nabla_\theta J_q(\theta)
    .\]
    With $M$ independent lists and independent feedback draws conditional on $q$,
    let $\overline U=M^{-1}\sum_{m=1}^M\widehat U_m$.
    The group-mean update satisfies
    \[
        \E\left[\frac1M\sum_{m=1}^M
            (\widehat U_m-\overline U)\nabla_\theta\log\pi_\theta(\sigma_m\mid x)
            \,\middle|\,q\right]
        =\left(1-\frac1M\right)\nabla_\theta J_q(\theta)
    .\]
\end{proposition}
\begin{proof}
    The first identity follows from $\E[Z_v\mid q,\sigma]=p_v(\sigma)$;
    the observation indicators need not be independent within a list.
    Substituting this conditional expectation into the score-function estimator
    gives the usual policy-gradient identity \citep{reinforce}:
    \[
        \E[\widehat U\nabla_\theta\log\pi_\theta\mid q]
        =\sum_\sigma U_q(S(\sigma))\nabla_\theta\pi_\theta(\sigma\mid x)
        =\nabla_\theta J_q(\theta)
    .\]
    Also $\E[\nabla_\theta\log\pi_\theta\mid q]=0$, so the independent baseline contributes zero.
    For the group mean, terms pairing $\widehat U_j$ with the score of list $m\ne j$
    have expectation zero by independence.
    The remaining term subtracts $M^{-1}\nabla_\theta J_q(\theta)$ from each update.
\end{proof}

For the $M=8$ samples used in our experiments, group-mean centering therefore
multiplies the expected policy gradient by $7/8$.
When every proposed node is judged, $p_v=1$ and one proposal needs at most $B$ node labels;
labels outside the proposal are unnecessary for this estimator.
Randomly judging fewer proposed nodes also gives an unbiased estimator when the
observation probabilities are known and positive, although small probabilities can increase variance.
The judgments must reveal the stated utility, or an unbiased estimate conditional on
observing that node; unjudged nodes are not automatically negative examples.
The same score-function identity also applies to directly observed set-level rewards:
additivity is needed only to assemble the reward from individual node judgments.

Exact recall requires the normalizer $\abs{A(q)}$ in addition to the labels of proposed nodes.
If the total number of relevant nodes is unknown, those labels alone do not identify recall. A policy can instead use the observable reward $|S\cap A(q)|/20$, where $S$ contains at most 20 proposed nodes and each selected valid answer contributes $1/20$. Unlike recall, this objective does not normalize by the query's total answer count, so it weights queries differently. It equals precision when exactly 20 nodes are returned, and keeps the denominator fixed when fewer nodes are returned. Supervised losses can also learn from judged nodes and pairs.

\section{Evaluation Metrics}
\label{app:metrics}

We evaluate retrieval with standard top-$k$ metrics. Let $R_k(q)$ denote the top-$k$ retrieved nodes for query $q$.

\paragraph{Hit@$k$ and Recall@$k$.}
For any cutoff $k$,
\[
\mathrm{Hit@}k
=\frac{1}{|\mathcal{Q}|}\sum_{q\in\mathcal{Q}} \mathbf{1}\{A(q)\cap R_k(q)\neq \emptyset\},
\qquad
\mathrm{Recall@}k
=\frac{1}{|\mathcal{Q}|}\sum_{q\in\mathcal{Q}}\frac{|A(q)\cap R_k(q)|}{|A(q)|}.
\]
In the experiments, we report Hit@1, Hit@5, and Recall@20.

\paragraph{MRR.}
Let $\mathrm{rank}(q)=\min\{i:v_i\in A(q)\}$ be the rank of the first correct answer in the returned list, with $\mathrm{rank}(q)=\infty$ and reciprocal rank zero when no correct answer is returned. Then
\[
\mathrm{MRR}=\frac{1}{|\mathcal{Q}|}\sum_{q\in\mathcal{Q}}\frac{1}{\mathrm{rank}(q)}.
\]

\paragraph{Hit@any and Recall@any.}
To evaluate candidate-set quality independent of the final ordering, we also report metrics over the full retrieved set $R(q)$:
\[
\mathrm{Hit@any}=\frac{1}{|\mathcal{Q}|}\sum_{q\in\mathcal{Q}} \mathbf{1}\{A(q)\cap R(q)\neq \emptyset\},
\qquad
\mathrm{Recall@any}=\frac{1}{|\mathcal{Q}|}\sum_{q\in\mathcal{Q}} \frac{|A(q)\cap R(q)|}{|A(q)|}.
\]

\section{Implementation Details}
\label{app:implementation_details}

This section describes the bounded-graph selection model used for the experiments. Graph expansion constructs its input; the learned policy selects and ranks nodes within that input. The configurations used in the reported experiments are specified in \cref{app:experiment_details}.

\subsection{Entity-Based Seeds and the Bounded Graph}
\label{app:seed_construction}

\paragraph{Finding named entities.}
We build a dictionary from each KG node's name fields and explicitly listed aliases. For instance, these fields include disease and gene names in PRIME, paper titles and author names in MAG, and product or brand names in AMAZON. A name mentioned in a description does not become an alias of the node containing that description. This prevents, for example, a product from being linked merely because its description mentions a different entity.

Both dictionary entries and queries are Unicode-normalized, case-folded, and split into alphanumeric tokens. We match contiguous query spans against the dictionary, discarding entries consisting entirely of stop words or the placeholder ``unknown.'' When one matching span is strictly contained in another, we retain the longer span. The remaining matches are ordered by span length, then by cosine similarity between the query and the linked node, with node ID breaking ties. Multiple aliases of one node consume only one seed slot. This is dictionary-based entity recognition and linking; it does not require a separately trained NER model or an LLM call.

\paragraph{Completing the core set.}
We retain at most $k_0$ distinct matched nodes. If the dictionary supplies fewer, we add unselected nodes in decreasing order of query--node cosine similarity until $|V_0|=k_0$. Thus $k_0$ is the \emph{total} number of core nodes, not an additional dense-retrieval budget. A query with no entity match uses cosine retrieval for all its seeds. Seeds may have any node type: a disease, author, or brand can be a useful starting point even when the answers must be drugs, papers, or products. The answer-type constraint is applied later, when selecting the output.

\paragraph{Retained topology and hop labels.}
We use the iterative scoring rule described in the main text to expand from $V_0$. Let $C_q^{(0)}=V_0$ and $C_q^{(h)}$ be the cumulative retained set after hop $h$, with at most $b_h$ new nodes. The final input is the exact induced KG subgraph $\widetilde G_q=G[C_q^{(H)}]$, including all original typed edges between retained nodes, not just the edges used to reach them. We write $C_q=C_q^{(H)}$ and record
\[
\eta_q(v)=\min\{h:v\in C_q^{(h)}\}.
\]
This is the \emph{first retrieval hop}. Filtering can postpone a node that is closer in the original KG, so $\eta_q(v)$ need not equal its shortest-path distance from $V_0$. We retain this label when reconstructing a compact cached graph.

The graph budget $B$ includes the core nodes and all retained intermediate nodes. Actual graph sizes may be smaller when a frontier is exhausted. Auxiliary encoder edges and the virtual node do not consume this budget. Candidate graphs are fixed before policy training; their memberships are not changed by the sampled actions. Hop budgets and optional first-hop entity-source priors are configuration choices; the latter replace the source similarity of an entity-matched seed at the first expansion only.

\subsection{Query-Adaptive Exphormer}
\label{app:exphormer}

The retrieved graph contains both candidate answers and nodes that explain their relation to the query. A graph encoder can use those relations without requiring the answer node itself to resemble the query. Exphormer~\citep{exphormer} is a useful backbone because its attention graph combines the observed topology with a small number of additional connections. These connections allow different branches of a neighborhood to exchange information, while the original KG edges retain their relation features. We adapt this architecture to retrieval through query-conditioned features, an optional query virtual node, and retrieval-hop encodings. \Cref{fig:query_exphormer} shows the resulting model.

\begin{figure}[t]
\centering
\includegraphics[width=\linewidth]{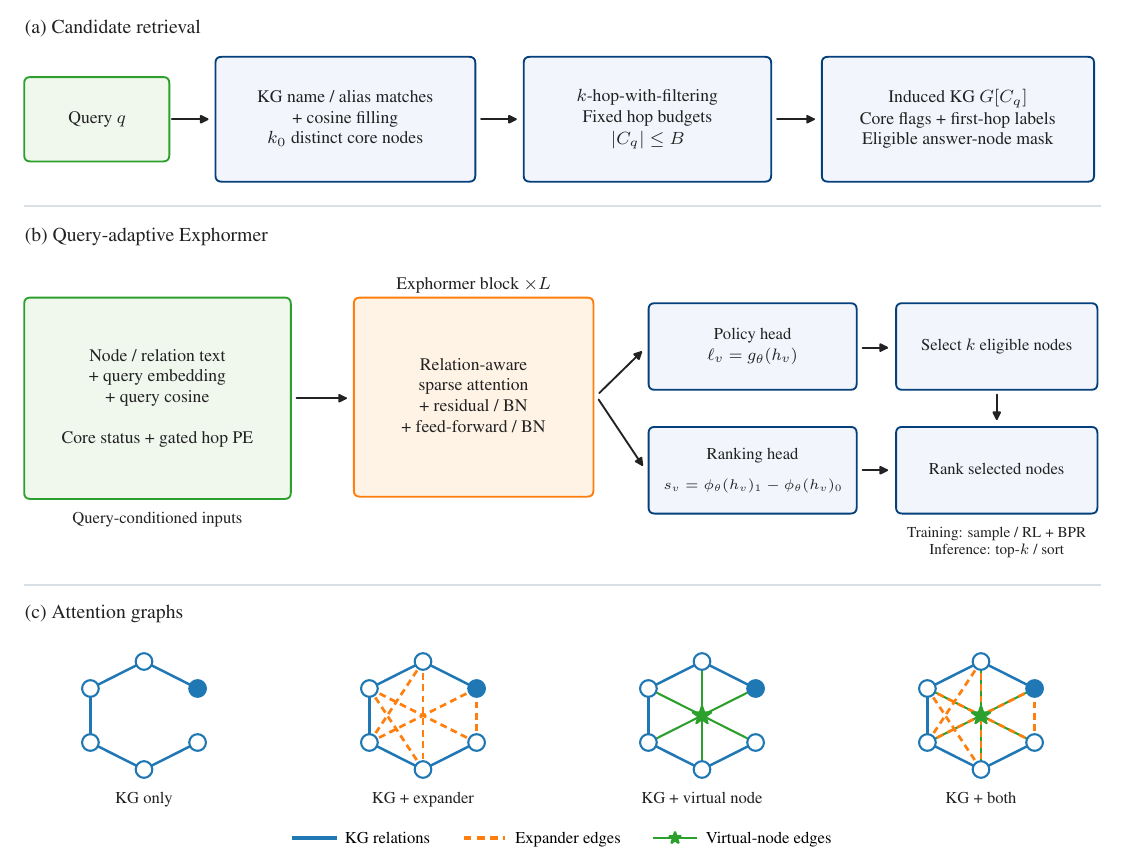}
\caption{\textbf{Query-adaptive Exphormer for \methodname{}.} Entity matches and cosine retrieval initialize a bounded, filtered neighborhood. Node and relation features are conditioned on the query; real nodes also receive core-status and first-retrieval-hop information. The same encoder pass supplies the policy and ranking heads. Expander edges and the query virtual node are optional attention connections. The four schematics at the bottom illustrate these choices; all retain the original KG edges. Auxiliary connections do not change the retrieved nodes or their hop labels.}
\label{fig:query_exphormer}
\end{figure}

\subsubsection{Input Features and Hop Encodings}

Let $z_q,z_v,z_r\in\mathbb R^{d_0}$ denote the fixed text embeddings of the query, node $v$, and relation type $r$. When dimensionality reduction is used, the same frozen transform is applied to all three. PCA fits use KG nodes, relation types, and training queries; validation and test query embeddings are transformed without refitting. Cosine similarities use normalized vectors, but the encoder receives the projected embeddings rather than replacing them with cosine scores alone.

The query-conditioned feature maps are
\begin{align}
u_v(q)&=F_v\!\left([z_v\,\|\,z_q\,\|\,\cos(z_v,z_q)]\right)+t_{\kappa(v)},\label{eq:node_query_fusion}\\
e_r(q)&=W_e F_e\!\left([z_r\,\|\,z_q\,\|\,\cos(z_r,z_q)]\right),\label{eq:edge_query_fusion}
\end{align}
where $F_v,F_e$ are learned affine maps, $\kappa(v)$ indicates whether $v$ is a core node, and $t_{\mathrm{core}},t_{\mathrm{other}}\in\mathbb R^{d_0}$ are learned status vectors. Bias terms are omitted from the notation. The node map is shared across real nodes, and the relation map across KG relation types. We compute one query-conditioned relation table per graph and look up its rows for the edges.

We project node features to hidden dimension $d$ and add the hop encoding:
\begin{equation}
h_v^{(0)}=W_x\operatorname{LN}(u_v(q))+\gamma p_{\eta_q(v)},
\label{eq:gated_hop_encoding}
\end{equation}
where $p_0,\ldots,p_H\in\mathbb R^d$ are learned vectors. They distinguish seeds from later discoveries and preserve their position in the filtered search. They do not encode an arbitrary node index or an absolute position in the KG. The scalar gate $\gamma$ is learned and initialized to zero; the hop vectors are initialized from $\mathcal N(0,0.02^2)$. This lets the model introduce hop information gradually. The gate is shared, rather than predicted separately for each query; query dependence enters through the input features and through the retrieved graph and its hop labels. Disabling hop encodings removes this term. An ungated additive encoding is obtained by fixing $\gamma=1$.

\subsubsection{The Attention Graph}

Let $n=|C_q|$ and $E_q$ be the edges of $\widetilde G_q$. We construct an attention multigraph from $E_q$ and two optional augmentations. The virtual-node indicator $\delta\in\{0,1\}$ and the number of expander cycles $c\geq0$ are hyperparameters. Thus the encoder can use KG edges alone, KG edges with expander connections, KG edges with a virtual node, or both augmentations. These are architectural choices, not assumptions that either augmentation must improve retrieval.

\paragraph{Expander connections.}
We sample $c$ random Hamiltonian cycles on the real nodes and add both directions of every cycle edge. For $n\geq2$, this contributes $2cn$ directed edge instances. Parallel edges are retained, including overlaps with KG edges; hence $2c$ is the added degree counting multiplicity, not necessarily the number of distinct new neighbors. Random cycles provide a sparse, expander-style communication pattern~\citep{exphormer}. They need not follow a KG relation or the retrieval frontier. Their endpoints are generated from the sorted original node IDs and a fixed seed, so the pattern is unchanged across epochs and batch order. All expander edges share a learned feature vector $t_{\mathrm{exp}}$, distinct from the KG relation features.

\paragraph{Query virtual node.}
When $\delta=1$, we add one auxiliary node $\star$ with
\begin{align}
h_\star^{(0)}&=W_\star\operatorname{LN}(z_q)+t_\star,\\
e_{v\to\star}(q)&=W_{\star e}\operatorname{LN}(z_q)+t_{\to\star},\qquad
e_{\star\to v}(q)=W_{\star e}\operatorname{LN}(z_q)+t_{\star\to}.
\end{align}
The two edge directions have separate learned type vectors. Every real node is connected to $\star$ in both directions, adding $2n$ directed edges. This gives any pair of real nodes a two-edge communication route and lets the auxiliary state aggregate evidence relative to the query. The virtual node has no retrieval-hop label and never receives an output score. Query information is already present in the real-node features, even when this node is disabled.

The attention graph is used only inside the encoder. In particular, expander edges are not interpreted as factual KG relations, and neither augmentation changes the candidate set, answer eligibility, or the recorded retrieval hops.

\subsubsection{Relation-Aware Sparse Attention}

Each layer uses $a$ attention heads, with head dimension $d_a=d/a$. For clarity, we omit the layer index and write the computation for one head. For an edge instance $j\to i$, the receiver supplies the query vector and the sender supplies keys and values:
\[
Q_i=W_Qh_i,\qquad K_j=W_Kh_j,\qquad V_j=W_Vh_j,
\qquad (E^K_{ji},E^V_{ji})=W_Ee_{ji}.
\]
Here $e_{ji}$ is the query-conditioned KG relation feature, the expander feature, or the directed virtual-node feature. We modulate the sender key by the edge feature and add an edge contribution to the value:
\begin{align}
a_{ji}&=\frac{Q_i^\top(K_j\odot E^K_{ji})}{\sqrt{d_a}},
&\alpha_{ji}&=\frac{\exp(a_{ji})}{\sum_{j'\to i}\exp(a_{j'i})},\label{eq:sparse_attention}\\
m_i&=\sum_{j\to i}\alpha_{ji}(V_j+E^V_{ji}).\label{eq:edge_message}
\end{align}
The sums run over incoming edge instances, including parallel edges. Attention is normalized separately for each receiver and head. We concatenate the head outputs and apply two residual blocks:
\begin{align}
\bar h_i&=\operatorname{BN}_1\!\left(h_i+\operatorname{Dropout}([m_i^1\|\cdots\|m_i^a])\right),\\
h_i'&=\operatorname{BN}_2\!\left(\bar h_i+\operatorname{Dropout}\!\left(W_2\operatorname{Dropout}(\operatorname{ReLU}(W_1\bar h_i))\right)\right).
\end{align}
The feed-forward block has width $2d$. The implementation uses batch normalization after each residual addition, with running statistics at evaluation. The $Q,K,V$ and edge-modulation maps have no biases. Head concatenation has no additional output projection in the configurations described here. Edge features remain fixed through the stack, while each layer learns its own modulation map. We use two attention heads and edge-feature dimension 16; hidden dimension, depth, and dropout are configuration choices listed in \cref{app:experiment_details}.

\paragraph{Cost and readout.}
The attention graph has $n+\delta$ nodes and $|E_q|+2cn+2\delta n$ edge instances. At fixed width and depth, attention costs are linear in this edge count, rather than quadratic in $n$ as in all-pairs attention. The augmentation adds only a linear number of edges; it does not make a dense induced KG sparse. All retained real nodes, including ineligible intermediate nodes, participate in message passing.

After $L$ layers we discard the virtual-node representation. Two separate MLPs map each real-node representation to a policy logit and a ranking score:
\[
\ell_v=g_\theta(h_v^{(L)}),\qquad
s_v=\phi_\theta(h_v^{(L)})_1-\phi_\theta(h_v^{(L)})_0.
\]
Each head has a hidden layer of width $2d$, LeakyReLU, and dropout. The policy head has one output; the ranking head has two, whose difference defines $s_v$. Both see the query through the shared encoder. Only nodes in the eligible answer-type set $\mathcal E_q\subseteq C_q$ can be selected for output. Eligibility follows the dataset's candidate-node restriction and does not use the query's answer labels. Core nodes are neither automatically included nor excluded. No second encoder pass is needed after selection.

\subsection{Policy Sampling and the Training Objective}
\label{app:policy_objective}

\paragraph{Sampling without replacement.}
Write $k_q=\min(k,|\mathcal E_q|)$. During training, the policy samples an ordered list $\tau=(v_1,\ldots,v_{k_q})$ of distinct eligible nodes. We draw independent standard Gumbel variables and take the top $k_q$ values of $\ell_v/\tau_s+g_v$, where $\tau_s>0$ is the sampling temperature. This is discrete Gumbel-top-$k$ sampling, not a continuous Gumbel--Softmax relaxation. Its ordered sample follows the Plackett--Luce distribution, whose log-probability is
\begin{equation}
\log\pi_\theta(\tau\mid q)
=\sum_{t=1}^{k_q}\left[
\frac{\ell_{v_t}}{\tau_s}
-\log\sum_{u\in\mathcal E_q\setminus\{v_1,\ldots,v_{t-1}\}}
\exp\!\left(\frac{\ell_u}{\tau_s}\right)\right].
\label{eq:pl_logprob}
\end{equation}
The denominator excludes earlier selections. Summing categorical log-probabilities from the original, unmasked softmax would give a different objective. All $M$ samples for a query use the same encoder output, with independent sampling noise; the word trajectory here refers to the ordered selections, not to an additional traversal of the KG.

\paragraph{Coverage reward.}
For $S(\tau)=\{v_1,\ldots,v_{k_q}\}$, the reward is
\[
r(\tau,q)=\frac{|S(\tau)\cap A(q)|}{|A(q)|},
\]
with reward zero for an empty answer set. The denominator is the full answer set, including answers outside the bounded graph. We do not reward a model merely for selecting all answers that survived preprocessing. This reward is Recall@any over the selected set; when the output budget is 20 it is also Recall@20. It is independent of the order supplied by the ranking head. This is the reward used for the main results; \cref{app:fixed_denominator_reward} evaluates correct selected nodes divided by 20 as an alternative that does not require the total answer count.

For each query, we sample $M$ trajectories and use their mean reward as a baseline. With $\bar r_q=M^{-1}\sum_m r_q^{(m)}$ and stop-gradient operator $\operatorname{sg}$, the minibatch policy loss is
\begin{equation}
\mathcal L_{\mathrm{RL}}
=-\frac{1}{|\mathcal B|M}\sum_{q\in\mathcal B}\sum_{m=1}^M
\operatorname{sg}(r_q^{(m)}-\bar r_q)
\log\pi_\theta(\tau_q^{(m)}\mid q).
\label{eq:group_policy_loss}
\end{equation}
This is a group-centered REINFORCE update~\citep{reinforce}; it uses neither a learned critic nor a clipped policy-ratio objective. Because the mean includes the current sample, it is not a leave-one-out baseline. If all sampled rewards agree, the query contributes no policy gradient. With $M=1$ this baseline would always give zero, so single-trajectory comparisons must disable it.

\paragraph{Auxiliary ranking loss.}
Within sampled set $S_q^{(m)}$, define $P_q^{(m)}=S_q^{(m)}\cap A(q)$ and $N_q^{(m)}=S_q^{(m)}\setminus A(q)$. For samples with both positive and negative nodes, we use all their positive--negative pairs:
\begin{equation}
\mathcal L_{\mathrm{BPR}}
=\frac{1}{|\mathcal I|}\sum_{(q,m)\in\mathcal I}
\frac{1}{|P_q^{(m)}||N_q^{(m)}|}
\sum_{p\in P_q^{(m)}}\sum_{n\in N_q^{(m)}}
-\log\sigma(s_p-s_n),
\label{eq:sampled_bpr}
\end{equation}
where $\mathcal I$ contains the valid query--sample pairs in the minibatch. Samples without a positive--negative pair are omitted; the loss is zero if $\mathcal I$ is empty. Thus every valid sample has equal weight regardless of its number of pairs. BPR differentiates through the ranking head and shared encoder, but not through the discrete selection indices. The policy head receives the REINFORCE gradient. In particular, this auxiliary loss is computed on the sampled output, not on every node in $C_q$; direct BPR over all candidates is a separate ablation.

\paragraph{Optimization and inference.}
We minimize $\mathcal L_{\mathrm{RL}}+\lambda_{\mathrm{BPR}}\mathcal L_{\mathrm{BPR}}$ with AdamW, cosine learning-rate scheduling with warmup, and global gradient-norm clipping. The full-model configurations use $M=8$, $k=20$, $\tau_s=1$, $\lambda_{\mathrm{BPR}}=1$, weight decay $10^{-5}$, and clipping threshold 1; learning rate, dropout, and duration vary by dataset. Validation uses deterministic top-$k$ policy selection followed by ranking with $s_v$. We retain the earliest checkpoint attaining the maximum validation Recall@20 and use that checkpoint for all test metrics. Test performance does not select the checkpoint.

The policy determines membership, while the scoring head determines the order of that membership. With a fixed output size, changing only the ranking cannot change Recall@20, though it can change Hit@1, Hit@5, and MRR. Neither head can select an answer excluded by the bounded-graph construction.

\begingroup
\raggedbottom
\subsection{Algorithm Pseudocode: Training and Inference}
\label{app:main_algorithms}

\Cref{alg:seeder_prepare} gives the shared preparation and encoding routines. Training caches only the fixed graph inputs, not the learned representations. \Cref{alg:seeder_training} performs one parameter update per minibatch; \cref{alg:seeder_inference} uses a single deterministic selection. All three use the same seed rule, induced topology, answer-type mask, and encoder augmentations. The $k$ output nodes are distinct, and fewer are returned only when fewer than $k$ eligible nodes are available.

\begin{algorithm}[H]
\caption{\methodname{}: shared graph preparation and encoding}
\label{alg:seeder_prepare}
\begin{algorithmic}[1]
\Require KG $G$; frozen embeddings; seed budget $k_0$; hop budgets $b_{1:H}$; encoder options $(c,\delta)$
\Function{Prepare}{$q$}
  \State Encode $q$ with the frozen text encoder and transform to obtain $z_q$ (or load its cache)
  \State Match normalized query spans to KG names and aliases; suppress contained spans
  \State Order matches by span length, query--node cosine, then node ID
  \State $V_0\gets$ first $k_0$ distinct matched nodes
  \State Fill $V_0$ to $k_0$ with unselected nodes in decreasing cosine order
  \State $C_q^{(0:H)}\gets\textsc{K-hop-with-filtering}(G,q,V_0,b_{1:H})$
  \State $C_q\gets C_q^{(H)}$; $\widetilde G_q\gets G[C_q]$
  \State $\eta_q(v)\gets\min\{h:v\in C_q^{(h)}\}$ for each $v\in C_q$
  \State $\mathcal E_q\gets$ nodes in $C_q$ of eligible answer types
  \State \Return $(z_q,\widetilde G_q,V_0,\eta_q,\mathcal E_q)$
\EndFunction
\Function{Encode}{$D_q;\theta$}
  \State Fuse node/relation embeddings with $z_q$ and their query cosine scores
  \State Add core-status vectors; project nodes and add gated hop embeddings
  \State Add $c$ fixed bidirectional cycles and, if $\delta=1$, the query virtual node
  \State Apply $L$ sparse attention blocks to the augmented graph
  \State Discard the virtual-node state; compute policy logits $\ell_v$ and ranking scores $s_v$
  \State \Return $(\ell,s)$ on real nodes
\EndFunction
\end{algorithmic}
\end{algorithm}

\begin{algorithm}[H]
\caption{\methodname{} inference}
\label{alg:seeder_inference}
\begin{algorithmic}[1]
\Require Query $q$; trained $\theta_{\mathrm{best}}$; KG and preparation settings; output budget $k$
\State Set evaluation mode and disable gradients
\State $D_q\gets\Call{Prepare}{q}$
\State $(\ell,s)\gets\Call{Encode}{D_q;\theta_{\mathrm{best}}}$
\State $k_q\gets\min(k,|\mathcal E_q|)$
\State $S\gets\operatorname{Top}_{k_q}(\mathcal E_q;\ell_v)$
\State \Return $R(q)\gets$ nodes of $S$ ordered by decreasing $s_v$
\end{algorithmic}
\end{algorithm}

\endgroup

\begin{algorithm}[H]
\caption{\methodname{} training on bounded candidate graphs}
\label{alg:seeder_training}
\begin{algorithmic}[1]
\Require Training queries with answers $A(q)$; validation queries; $M$, $k$, $\tau_s$, $\lambda_{\mathrm{BPR}}$
\State Cache $D_q\gets\Call{Prepare}{q}$ for each training and validation query
\State Initialize $\theta$, optimizer, and scheduler; $v_{\mathrm{best}}\gets-\infty$
\For{each epoch}
  \For{each minibatch $\mathcal B$}
    \State Set training mode; clear parameter gradients
    \ForAll{$q\in\mathcal B$}
      \State $(\ell,s)\gets\Call{Encode}{D_q;\theta}$
      \State $k_q\gets\min(k,|\mathcal E_q|)$
      \For{$m=1,\ldots,M$}
        \State Draw $g_v^{(m)}\sim\operatorname{Gumbel}(0,1)$ independently for $v\in\mathcal E_q$
        \State $\tau_q^{(m)}\gets\operatorname{Top}_{k_q}(\mathcal E_q;\ell_v/\tau_s+g_v^{(m)})$ in score order
        \State Compute $\log\pi_\theta(\tau_q^{(m)}\mid q)$ using \cref{eq:pl_logprob}
        \State $r_q^{(m)}\gets |S(\tau_q^{(m)})\cap A(q)|/\max(1,|A(q)|)$
        \State Form positive--negative ranking pairs within $S(\tau_q^{(m)})$
      \EndFor
      \State $\bar r_q\gets M^{-1}\sum_m r_q^{(m)}$
    \EndFor
    \State Compute minibatch losses using \cref{eq:group_policy_loss,eq:sampled_bpr}
    \State Backpropagate $\mathcal L_{\mathrm{RL}}+\lambda_{\mathrm{BPR}}\mathcal L_{\mathrm{BPR}}$
    \State Clip gradient norm; update $\theta$ once; advance the scheduler
  \EndFor
  \State $v\gets$ mean validation Recall@20 using \cref{alg:seeder_inference}
  \If{$v>v_{\mathrm{best}}$}
    \State $\theta_{\mathrm{best}}\gets\theta$; $v_{\mathrm{best}}\gets v$
  \EndIf
\EndFor
\State \Return $\theta_{\mathrm{best}}$
\end{algorithmic}
\end{algorithm}

\section{Baseline Methods}
\label{app:baselines}

We compare \methodname{} against dense retrieval, heuristic graph expansion, and LLM-based agentic graph retrieval baselines. Most KG-RAG baselines and their reported results are taken from \citet{stark} and \citet{graphflow}, which together evaluate a wide range of models for the STARK benchmark.

\paragraph{Dense Retriever.}
Dense retrieval ranks each node independently by embedding similarity to the query. Following the standard bi-encoder retrieval paradigm~\citep{karpukhin2020dense}, the query and node text are encoded into dense vectors and nodes are retrieved by cosine similarity or inner product. 

For MiniLM-based experiments, we use \texttt{sentence-transformers/}\allowbreak\texttt{all-MiniLM-L6-v2} as a lightweight dense text encoder. This model belongs to the Sentence Transformers \texttt{all-*} family, which is trained on more than one billion training pairs and designed as a general-purpose sentence embedding model.

In the STARK-PRIME encoder comparison, we also evaluate OpenAI \texttt{text-embedding-ada-002} and Qwen3-Embedding-4B~\citep{zhang2025qwen3}. This comparison tests how well one-shot semantic similarity can retrieve answer nodes without using graph expansion.

\paragraph{G-Retriever.}
Following G-Retriever~\citep{gretriever}, which formulates graph retrieval as Prize-Collecting Steiner Tree selection~\citep{archer2011improved} over a textual graph, we implement only its non-fine-tuned retrieval stage. Following \citet{graphflow}, for each query we retrieve dense seed nodes using the existing node embeddings, construct a bounded 2-hop ego graph, assign node prizes from node-query similarity and edge prizes from head, relation, and tail similarity, and solve the PCST objective using \texttt{pcst\_fast} \citep{hegde2014fast,hegde2015nearly,goemans1995general}. The selected subgraph nodes are ranked by their embedding relevance plus prize bonuses; no graph prompt, generation stage, or soft-prompt/LLM fine-tuning is used.

\paragraph{SubgraphRAG.}
Following SubgraphRAG~\citep{li2025simple}, which retrieves query-relevant KG subgraphs using a triple-scoring mechanism before LLM reasoning, we implement a retrieval-only variant of its subgraph selection stage. Starting from dense embedding seed nodes, we build a bounded local 2-hop subgraph and score each triple by a weighted combination of head-node, relation-type, and tail-node similarity to the query. The endpoints of the highest-scoring triples, together with the original seeds, are returned as the retrieved nodes. This replaces SubgraphRAG's learned MLP triple scorer and excludes the downstream LLM reasoning component.

\paragraph{Beam Traversal.}
Motivated by BeamQA~\citep{atif2023beamqa}, which uses beam search to execute multi-hop knowledge graph query-answering paths, we implement a non-learned beam-search traversal baseline. The method starts from the top 3 dense seed nodes and expands the KG hop by hop for 2 hops, scoring each transition by destination-node similarity, relation-type similarity, previous path score, and a small hop penalty. At each hop, it keeps only the best frontier nodes under a fixed beam width and finally ranks all reached nodes by their best traversal score.

\paragraph{A*-Style Search.}
Inspired by A*Net~\citep{zhu2023astarnet}, which learns an A*-like priority function for scalable path-based KG reasoning, we implement a non-learned priority-queue traversal baseline. Starting from the top 3 dense seeds, each frontier state is prioritized by its accumulated path score plus a heuristic based on node-query embedding similarity. The traversal expands only the highest-priority states up to a fixed budget and depth, scoring transitions with the same node, relation, path, and hop-penalty terms as beam traversal. Unlike A*Net, the priority function is fixed and no edge or node selector is trained.

\paragraph{Personalized PageRank.}
Following personalized PageRank for query-biased graph importance~\citep{jeh2003scaling}, we implement an embedding-seeded local PPR retriever. The restart distribution is placed on the top 3 dense seed nodes, weighted by their query similarity, and PPR is run on a bounded 2-hop ego graph with restart probability $0.15$ for up to 20 iterations. Final relevance is a weighted combination of normalized PPR score and dense node-query similarity, and candidate nodes are ranked by this combined score.

\paragraph{PPR + MMR Reranker.}
Following Maximal Marginal Relevance~\citep{carbonell1998mmr}, we add a diversity-aware reranker on top of the PPR candidate pool. We first form a pool of the top 250 graph-expanded candidates using the PPR relevance score, then greedily select 20 nodes by maximizing a linear trade-off between query relevance and dissimilarity to already selected nodes. Relevance is the normalized PPR/dense score, redundancy is cosine similarity between node embeddings, and we use MMR weight $\lambda = 0.78$.

\paragraph{STARK VSS and Multi-VSS, with and without reranking.}
We include the vector-similarity baselines from STARK~\citep{stark}. VSS is a dense retriever that uses the OpenAI \texttt{text-embedding-ada-002} text embedding model. Multi-VSS represents each candidate node with multiple chunks or vectors and aggregates the query--chunk similarities. These baselines measure the effect of stronger vector indexing over text-rich semi-structured knowledge bases.

We also compare with STARK's LLM reranking baselines. These methods first retrieve candidates using VSS and then rerank the top candidates using a stronger language model, such as Claude 3 or GPT-4. Since reranking operates only on the initially retrieved candidates, these methods can improve rank-sensitive metrics but cannot recover answers missed by the first-stage retriever.

\paragraph{ToG.}
Think-on-Graph (ToG) is an LLM-agent-based graph traversal method for KG reasoning~\citep{ToG}. It uses an LLM to iteratively select graph neighbors and search for evidence relevant to the query. \citet{graphflow} adapt ToG for retrieval evaluation by returning retrieved node documents at each search step rather than directly generating an answer. Because running ToG over the full STARK graph is infeasible, the search is initialized from dense-retrieval seed nodes and constrained to the 2-hop neighborhood around the seed. The baseline is instantiated with both LLaMA3-8B-Instruct~\citep{llama3} and GPT-4o.

\paragraph{SFT.}
SFT is a supervised fine-tuning baseline built on top of the ToG-style LLM search agent. The agent is fine-tuned on valid retrieval trajectories so that it learns to imitate graph-search decisions that lead to relevant nodes. In the setup of \citet{graphflow}, SFT uses LLaMA3-8B-Instruct as the backbone and applies LoRA fine-tuning~\citep{lora} for efficiency.

\paragraph{PRM.}
The Process Reward Model (PRM) baseline follows the step-wise supervision paradigm of \citet{lightman2024let}. In contrast to standard supervised fine-tuning, which imitates complete successful trajectories, PRM learns a reward model that scores intermediate state-action pairs during graph search. Given a retrieval state $s_t$ and a candidate action $a_t$, the model predicts a scalar process reward $r_\theta(s_t,a_t)$, which is used to guide the agent toward more promising next-hop decisions. In the GraphFlow setup~\citep{graphflow}, PRM is implemented with LLaMA3-8B-Instruct as the backbone and trained using curated intermediate preferences over retrieval steps.

\paragraph{GraphFlow.}
GraphFlow is the main method proposed by \citet{graphflow}. It formulates KG retrieval as a multi-step decision process and trains an LLM-based retrieval policy with a GFlowNet-style objective~\citep{gflownet}. Instead of relying on explicit process-level labels, GraphFlow learns a flow estimator that factorizes terminal outcome rewards over intermediate retrieval states. This model also uses LoRA fine-tuning~\citep{lora} of LLaMA3-8B-Instruct \citep{llama3}.

\paragraph{Reranked agent baselines.}
For LLM-agent baselines, \citet{graphflow} also report variants with reranking. In these settings, the agent first retrieves a set of candidate nodes, and then the retrieved results are reranked by the LLM. The same LLM performs retrieval and reranking.

\section{Cross-Dataset K-Hop Analysis}
\label{app:cross_dataset_khop}

\subsection{K-Hop Expansion}
\label{app:k-hop-expansion}
\Cref{fig:cross_dataset_khop_expansion} extends the $k$-hop expansion analysis from \cref{fig:multi-hop-expansion-results} to all three STARK knowledge graphs. The same pattern appears across datasets: even when retrieval starts from only a few seed nodes, uniform neighborhood expansion grows rapidly and approaches a large fraction of the graph within a small number of hops. This growth is especially problematic for STARK-AMAZON and STARK-MAG, whose knowledge graphs contain over one million nodes. In these settings, a naive multi-hop expansion can quickly stop being a targeted retrieval step and instead become close to whole-graph enumeration.

\begin{figure}[H]
    \centering
    \includegraphics[width=\linewidth]{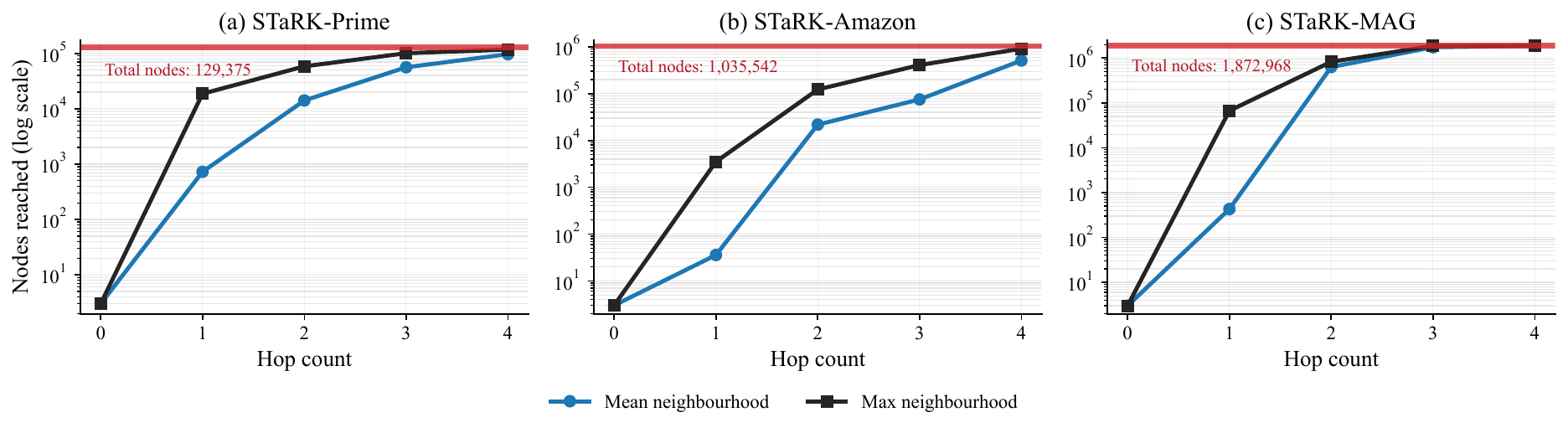}
    \caption{\textbf{Cross-dataset growth of naive $k$-hop expansion.}
    The number of nodes reached from dense seed nodes grows by orders of magnitude as the hop count increases. Blue curves show the mean neighborhood size and black curves show the maximum neighborhood size; the red horizontal line marks the total number of nodes in each knowledge graph. Across STARK-PRIME, STARK-AMAZON, and STARK-MAG, uniform expansion rapidly approaches graph-scale candidate sets, illustrating why first-stage retrieval needs selective, budgeted expansion rather than exhaustive $k$-hop neighborhoods.}
    \label{fig:cross_dataset_khop_expansion}
\end{figure}

The max-neighborhood curves show the worst-case behavior: for some seed nodes, one or two hops already expose tens of thousands to nearly all reachable nodes, depending on the dataset. The mean curves are lower but still grow by orders of magnitude on the log scale. STARK-MAG is particularly dense under this view, with the mean 2-hop expansion already reaching a very large portion of the graph. These trends motivate the bounded and learned expansion strategy used by \methodname{}: rather than expanding all neighbors at each hop, the retriever must decide which frontier nodes are worth following under a strict budget.

\subsection{Budgeted K-Hop Filtering Across Metrics}
\label{app:khop_filtering_metric_grid}

\paragraph{Setup.}
We further evaluate \textsc{K-hop-with-filtering} as a non-learned graph expansion baseline under different retrieved-node budgets. For each query, we first select dense seed nodes by cosine similarity and then run the frontier-filtered expansion procedure described in the main text. We compare the resulting retrieved sets against pure dense retrieval on all three STARK datasets, sweeping the final retrieval budget and reporting Hit@any, Recall@any, and MRR.

\begin{figure}[h]
    \centering
    \includegraphics[width=\linewidth]{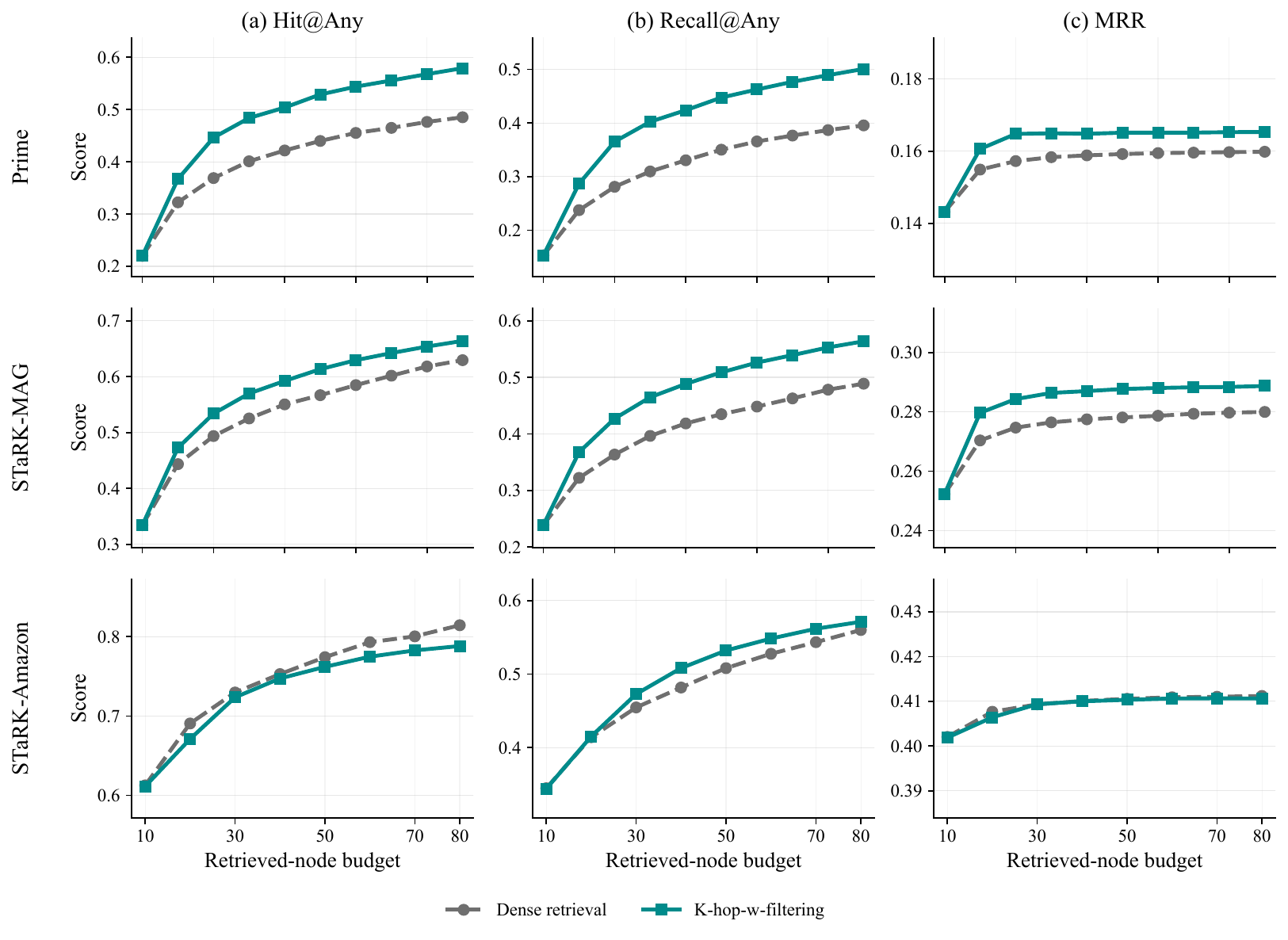}
    \caption{\textbf{Dense retrieval versus budgeted \textsc{K-hop-with-filtering} across datasets and metrics.}
    Each row corresponds to a STARK dataset and each column to a retrieval metric. The curves compare pure dense retrieval with the same initial dense seeds followed by budgeted frontier expansion. Filtering over local graph neighborhoods improves coverage-oriented metrics most clearly on STARK-PRIME and STARK-MAG, while gains on STARK-AMAZON are smaller and more metric-dependent.}
    \label{fig:dense_vs_khop_filtering_metric_grid}
\end{figure}

\paragraph{Analysis.}
\Cref{fig:dense_vs_khop_filtering_metric_grid} shows that budgeted graph expansion is most beneficial when answer nodes are reachable through useful local structure but are not ranked highly by dense similarity alone. On STARK-PRIME and STARK-MAG, \textsc{K-hop-with-filtering} consistently improves Hit@any and Recall@any across budgets, with the largest gains at small and medium budgets where the candidate set is still compact. MRR improves only modestly, suggesting that the main benefit of filtering is candidate discovery rather than fine-grained ranking. On STARK-AMAZON, dense retrieval is already strong and often matches or exceeds filtered expansion on Hit@any, while Recall@any remains competitive for the filtered method. This limits how much graph-discovery algorithms can improve on the initial dense retrieval and explains the limited gains of graph-search algorithms over the dense retriever in \cref{tab:main_results}. This pattern supports the role of \textsc{K-hop-with-filtering} as a useful but limited graph-aware baseline: it can recover additional relevant nodes from local neighborhoods, yet its greedy similarity-based frontier decisions do not reliably optimize ranking quality.

\paragraph{Scaling to larger candidate graphs.}
We extend this comparison to much larger candidate graphs, evaluating budgets from 50 to 2,000 on the full official validation splits (\cref{fig:large_budget_khop}). Filtered $k$-hop expansion retains higher Hit@any and Recall@any than dense retrieval on PRIME and MAG throughout this range, although the gap narrows on MAG as the budget grows. AMAZON behaves differently: using twenty initial seeds improves coverage over the five-seed variant, but the resulting $k$-hop retriever exceeds dense Recall@any only at budgets up to 350. At 2,000 nodes, its Recall@any is 83.08\%, compared with 89.10\% for dense retrieval. This suggests that broader seeding helps reach useful product neighborhoods, while continued local expansion does not necessarily find answers as effectively as global similarity ranking. The AMAZON results therefore show that increasing the graph budget alone does not guarantee an advantage over dense retrieval.

\begin{figure}[t]
    \centering
    \includegraphics[width=\linewidth]{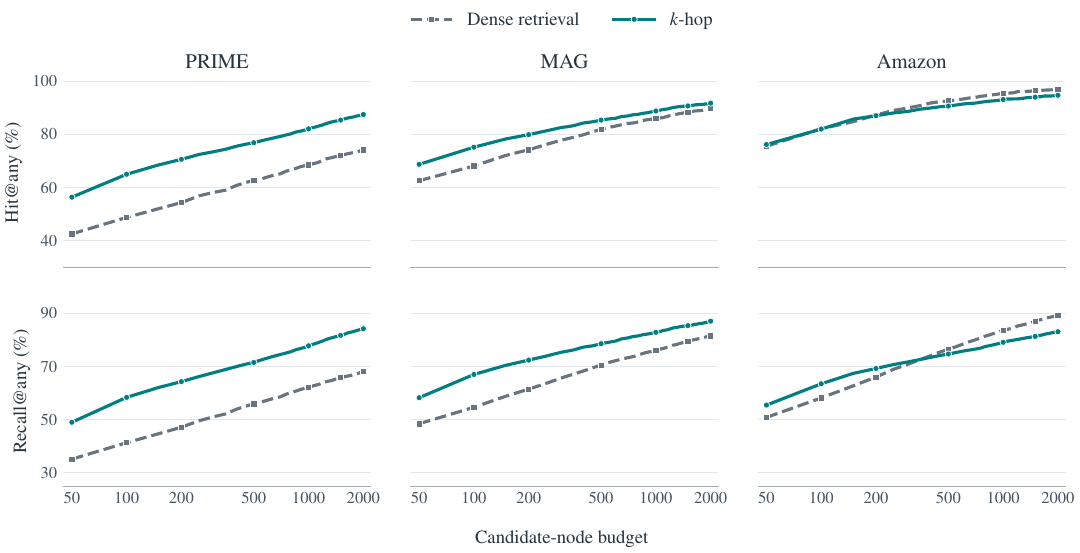}
    \caption{Dense retrieval and filtered $k$-hop expansion at larger candidate budgets. Results cover all 2,241 PRIME, 2,665 MAG, and 1,548 AMAZON validation queries using frozen MiniLM embeddings. Rows show Hit@any and Recall@any over eligible retained nodes, without learned reranking or a 20-output cap. The sweep uses 40 cumulative expansion steps: five initial entity/cosine seeds for PRIME and MAG and twenty for AMAZON, followed by enough nodes to reach a budget of 50 at the first step and at most 50 additional nodes per step thereafter. Seeds count toward the total budget; actual counts can be smaller. Curves show query averages at 40 budgets from 50 to 2,000 on a logarithmic axis.}
    \label{fig:large_budget_khop}
\end{figure}

\section{Training Stability}
\label{app:training_stability}

We evaluate \methodname{} on STARK-PRIME with
MiniLM-L6-v2 embeddings reduced to 256 dimensions. Ten runs (seeds 0--9) use
500-node candidate sets, five entity-first/cosine-fill seeds, a query virtual
node, and three expander cycles with a fixed graph-layout seed. The GNN has
hidden dimension 24, three layers, and dropout 0.1. Each run trains for ten
epochs. 

We hold out 20\% of the original training split using a fixed permutation
(seed 2026), yielding 4,930 training and 1,232 internal-validation queries.
The original validation split (2,241 queries) serves as the held-out evaluation
split; the original test split is not used. PCA is fitted using nodes,
relations, and the new training queries only. Hop allocation is selected on
internal validation. Each run selects the earliest checkpoint maximizing
internal-validation Recall@20; held-out scores do not select checkpoints.

\Cref{fig:rl_stability} shows the ten learning trajectories. Supervised loss,
reward, and recall summarize optimization behavior; the signed policy-gradient
loss need not decrease monotonically. The selected checkpoints achieve held-out
Recall@20 of $60.87\pm0.46$ and MRR of $41.41\pm0.67$ (mean and sample standard
deviation across ten seeds, in percentage points). 

\begin{figure}[t]
    \centering
    \includegraphics[width=\linewidth]{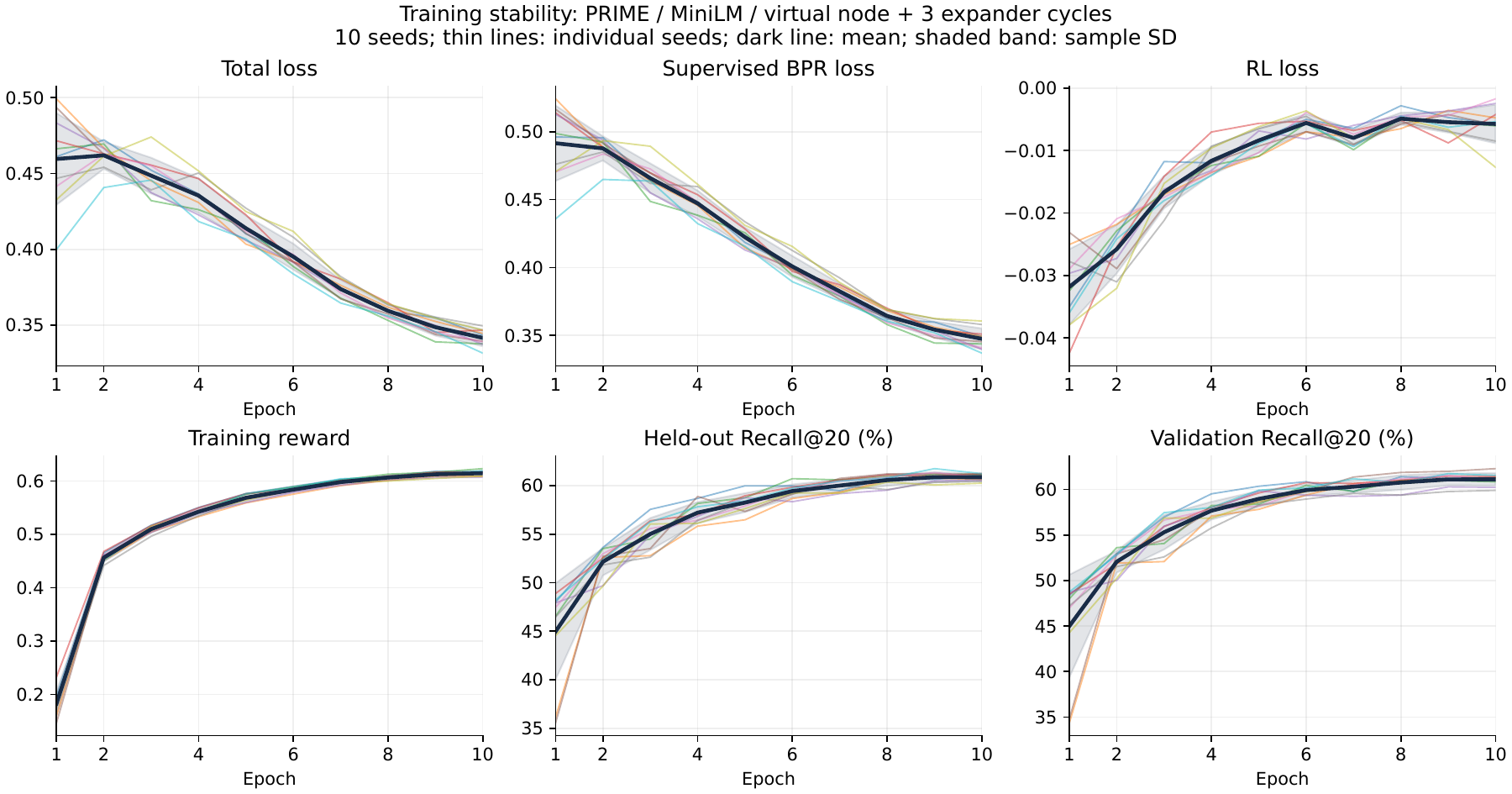}
    \caption{\textbf{MiniLM training stability on STARK-PRIME.}
    Ten seeds and ten epochs per seed, with a virtual node and three expander
    cycles. Curves show total, supervised BPR, and RL losses, training reward,
    internal-validation Recall@20, and held-out Recall@20. Held-out evaluation
    uses the original validation split, not the original test split.}
    \label{fig:rl_stability}
\end{figure}

\begin{table}[t]
\centering
\caption{\textbf{Validation-selected checkpoints across ten seeds.}
Mean $\pm$ sample standard deviation; all metrics are multiplied by 100.
The held-out split is the original validation split.}
\label{tab:stability_selected}
\begin{tabular}{lrrrr}
\toprule
Split & Hit@1 & Hit@5 & MRR & Recall@20 \\
\midrule
Validation & $31.05 \pm 1.00$ & $55.07 \pm 1.25$ & $41.78 \pm 0.85$ & $61.21 \pm 0.69$ \\
Held-out & $31.38 \pm 0.82$ & $54.23 \pm 0.79$ & $41.41 \pm 0.67$ & $60.87 \pm 0.46$ \\
\bottomrule
\end{tabular}

\end{table}

Across 100 epoch--seed checkpoint pairs, validation and held-out metrics are
strongly associated (\cref{fig:validation_heldout_correlation}). These pooled correlations
are descriptive: checkpoints from the same seed are dependent and share a
training trajectory, so the 100 pairs are not 100 independent experimental
replicates. Correlation therefore does not establish a general guarantee that
validation rankings transfer to unseen datasets or other configurations.

\begin{figure}[t]
\centering
\includegraphics[width=\linewidth]{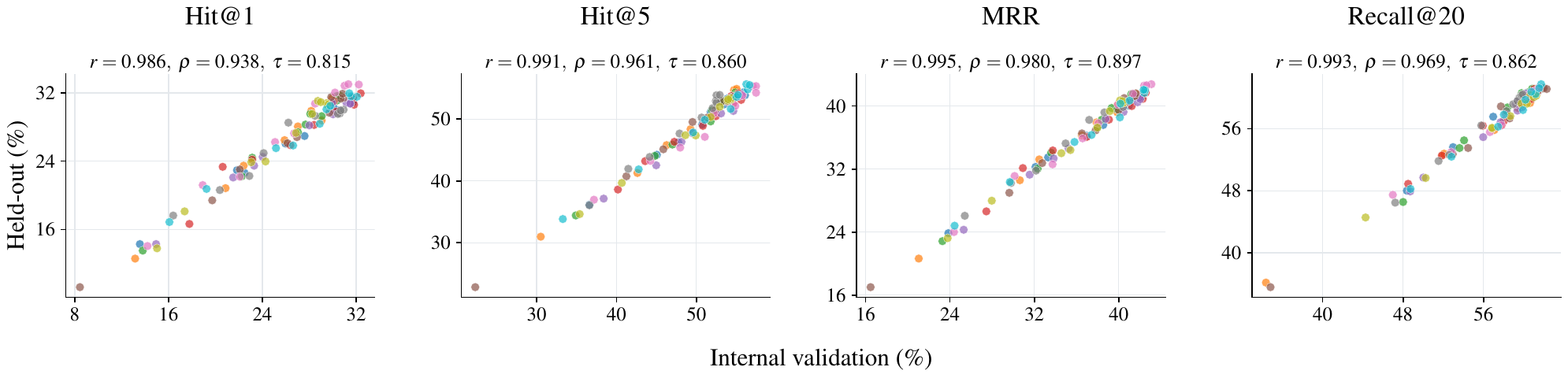}
\caption{Internal-validation and held-out agreement on PRIME across ten MiniLM runs and ten epochs per run. Colors identify seeds; panels report Pearson $r$, Spearman $\rho$, and Kendall $\tau$. The 100 checkpoints are dependent observations along training trajectories.}
\label{fig:validation_heldout_correlation}
\end{figure}

\section{Training without total answer counts}
\label{app:fixed_denominator_reward}

Recall rewards require the total number of answers to each query, in addition to judgments on the proposed nodes. When the total is unavailable, we can train with the fixed-denominator reward
\[
 r_{20}(S,q)=\frac{|S\cap A(q)|}{20},
\]
where $S$ is the set of at most 20 proposed nodes. Computing this reward only requires checking whether each proposed node is a valid answer; the full set $A(q)$ need not be enumerated. The denominator remains 20 even when fewer nodes are returned. This reward equals precision for an output set of size 20 and differs from recall in its weighting across queries.

We compare this reward with recall using matched Qwen configurations on STARK-PRIME, STARK-MAG, and STARK-AMAZON. Both conditions use frozen PCA-512 features, the same cached candidate graphs capped at 500 nodes, and five training seeds (0--4). The architectures and dataset-specific settings match the Qwen configurations in \cref{tab:current_hyperparameters,tab:current_retrieval}: PRIME uses expander edges without a virtual node, while MAG and AMAZON also use a virtual node. All other training settings are preserved, including eight trajectories, the group-mean baseline, and BPR weight 1. Within each run, the earliest checkpoint with maximum validation Recall@20 is evaluated on the official full test split. This comparison therefore removes the total-answer-count requirement from the training reward; checkpoint selection and evaluation still use benchmark answer counts.

\begin{table}[t]
\centering
\caption{Training-reward comparison with Qwen on the full STARK test splits. Values are percentages, reported as mean $\pm$ sample standard deviation over five seeds. Only the training reward changes; both conditions select checkpoints by validation Recall@20.}
\label{tab:fixed_denominator_reward}
\small
\setlength{\tabcolsep}{4pt}
\begin{tabular}{llrrrr}
\toprule
Dataset & Training reward & Hit@1 & Hit@5 & MRR & Recall@20 \\
\midrule
PRIME & Recall & $46.06 \pm 1.38$ & $74.34 \pm 0.55$ & $58.30 \pm 0.97$ & $79.43 \pm 0.27$ \\
PRIME & Correct nodes$/20$ & $46.53 \pm 2.45$ & $73.67 \pm 1.16$ & $58.55 \pm 1.83$ & $79.04 \pm 0.91$ \\
\midrule
MAG & Recall & $56.29 \pm 0.36$ & $75.45 \pm 0.37$ & $64.67 \pm 0.22$ & $75.13 \pm 0.16$ \\
MAG & Correct nodes$/20$ & $55.75 \pm 0.91$ & $75.16 \pm 0.31$ & $64.38 \pm 0.51$ & $75.19 \pm 0.26$ \\
\midrule
AMAZON & Recall & $46.76 \pm 0.51$ & $71.99 \pm 0.21$ & $58.07 \pm 0.44$ & $62.33 \pm 0.24$ \\
AMAZON & Correct nodes$/20$ & $46.33 \pm 1.15$ & $71.47 \pm 1.29$ & $57.60 \pm 1.02$ & $61.75 \pm 0.38$ \\
\bottomrule
\end{tabular}
\end{table}

The two rewards yield close mean Recall@20 values (\cref{tab:fixed_denominator_reward}). Replacing recall with correct nodes divided by 20 changes Recall@20 by $-0.39$, $+0.06$, and $-0.58$ percentage points on PRIME, MAG, and AMAZON, respectively. The fixed-denominator reward has higher Recall@20 variation across seeds on all three datasets. It slightly improves mean Hit@1 and MRR on PRIME, while both metrics decline on MAG and AMAZON. These results support the fixed-denominator reward as a practical option when training-time answer counts are unavailable, without showing a consistent improvement over recall. Because the BPR weight is unchanged, the comparison changes both query weighting and the magnitude of the policy reward relative to BPR; it does not isolate these effects.

\section{Effect of candidate graphs on MiniLM retrieval}
\label{app:minilm_candidate_graphs}
We compare separately trained MiniLM models using their original candidate graphs with models using frozen Qwen candidate graphs. Each dataset has five training seeds (0--4), a cap of 500 real candidate nodes, and at most 20 eligible outputs. The intervention replaces the ordered candidate and core-node IDs, retrieval hop labels, and induced knowledge-graph topology. Node, relation, and query features remain the same frozen MiniLM features in both conditions; no Qwen features are supplied to the learned model, and no feature transform is refitted.

\begin{figure}[t]
\centering
\includegraphics[width=\linewidth]{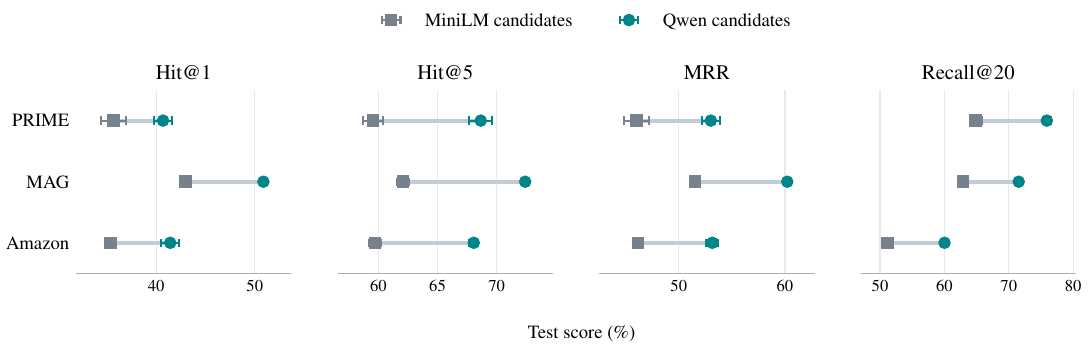}
\caption{MiniLM retrieval with MiniLM versus Qwen candidate graphs. Points show full-test means over five training seeds; horizontal error bars show sample standard deviations. Connected points compare the same dataset-specific MiniLM architecture and training settings under the two candidate sources. Both sources have a 500-node cap. The comparison changes candidate membership, core nodes, topology, and hop assignments together.}
\label{fig:minilm_candidate_graphs}
\end{figure}

\paragraph{Training protocol.}
All models use 256-dimensional MiniLM features, dropout 0.1, training/evaluation batch sizes 16/64, eight policy trajectories, and RL+BPR with auxiliary weight 1. PRIME uses hidden dimension 24, three layers, learning rate 0.001, and 15 epochs, with a query virtual node but no expander edges or hop embeddings. MAG uses hidden dimension 32, four layers, learning rate 0.001, and 15 epochs; AMAZON uses hidden dimension 64, three layers, learning rate 0.0003, and 10 epochs. MAG and AMAZON both use a virtual node, five expander cycles, and gated hop embeddings. The optimizer, initialization procedure, and training settings match the corresponding independent MiniLM reference. The earliest checkpoint attaining maximum validation Recall@20 is selected separately for each run, and all four metrics are evaluated on the official test split.

The Qwen graphs use five core nodes and hop quotas $(150,250,95)$ for PRIME, five and $(99,149,148,99)$ for MAG, and twenty and $(72,121,46,241)$ for AMAZON. The PRIME MiniLM reference instead uses nominal quotas $(200,200,100)$ under the same 500-node total cap. MAG and AMAZON retain the same nominal quotas across candidate sources, but their selected nodes and graph structure differ. The AMAZON reference in this experiment uses the matched 64-hidden-unit, ten-epoch configuration, rather than the 24-hidden-unit, fifteen-epoch configuration retained in \cref{tab:main_results}.

\begin{table}[t]
\centering\small
\caption{Full-test retrieval with matched MiniLM models and different candidate sources. Values are percentages, reported as mean $\pm$ sample standard deviation across five seeds.}
\label{tab:minilm_candidate_graphs}
\setlength{\tabcolsep}{4pt}
\begin{tabular}{llcccc}
\toprule
Dataset & Candidates & Hit@1 & Hit@5 & MRR & Recall@20\\
\midrule
PRIME & MiniLM & $35.67 \pm 1.23$ & $59.52 \pm 0.84$ & $46.03 \pm 1.15$ & $64.85 \pm 0.94$\\
PRIME & Qwen & $40.69 \pm 0.90$ & $68.65 \pm 0.96$ & $53.02 \pm 0.85$ & $75.91 \pm 0.63$\\
MAG & MiniLM & $42.99 \pm 0.42$ & $62.06 \pm 0.48$ & $51.52 \pm 0.38$ & $62.89 \pm 0.35$\\
MAG & Qwen & $50.87 \pm 0.29$ & $72.41 \pm 0.22$ & $60.21 \pm 0.19$ & $71.54 \pm 0.24$\\
AMAZON & MiniLM & $35.37 \pm 0.28$ & $59.72 \pm 0.53$ & $46.17 \pm 0.20$ & $51.18 \pm 0.25$\\
AMAZON & Qwen & $41.43 \pm 0.92$ & $68.06 \pm 0.37$ & $53.16 \pm 0.55$ & $60.02 \pm 0.18$\\
\bottomrule
\end{tabular}
\end{table}

\paragraph{Results.}
Using Qwen candidate graphs improves all four mean test metrics on all three datasets (\cref{fig:minilm_candidate_graphs,tab:minilm_candidate_graphs}). Recall@20 rises from 64.85 to 75.91 on PRIME, 62.89 to 71.54 on MAG, and 51.18 to 60.02 on AMAZON. Candidate-set oracle recall also increases, from 78.10 to 88.42, 78.00 to 88.36, and 76.34 to 83.18, respectively. Mean candidate counts are 470.34/480.22, 493.91/494.22, and 497.97/497.54 for MiniLM/Qwen graphs. Thus, the gains accompany improved answer coverage under the same cap, rather than a uniformly larger candidate set.

\section{Experiment details and hyperparameters}
\label{app:experiment_details}
This section specifies the experiments reported in \cref{sec:experiments}.

\paragraph{Training infrastructure.}
GPU training runs used a single NVIDIA H100 GPU partition with 40\,GB of device memory. The MAG and AMAZON ada-002 runs were trained locally on a MacBook Pro M3 chip. Training does not require distributed execution or model parallelism. The separate latency and memory evaluation protocol is described below.

\subsection{Training and model selection}
We use the official STARK train, validation, and test splits, with 2,801 PRIME, 2,665 MAG, and 1,642 AMAZON test queries. Each learned configuration is run with seeds 0--4. We train for the fixed epoch counts in \cref{tab:current_hyperparameters}, select the earliest checkpoint attaining maximum validation Recall@20 within each run, and evaluate that checkpoint on the full test split. All four reported metrics use the same checkpoint. Standard deviations describe variation across training seeds, not uncertainty over queries.

\begin{table}[t]
\centering\small
\caption{Main-table model and training hyperparameters. Dim. is the frozen embedding dimension after PCA; $h$ is the hidden dimension and $L$ the number of graph layers. VN denotes a query virtual node, E denotes five expander cycles, and H denotes gated hop embeddings.}
\label{tab:current_hyperparameters}
\setlength{\tabcolsep}{4pt}
\begin{tabular}{llrrrrrrl}
\toprule
Dataset & Encoder & Dim. & $h$ & $L$ & Dropout & LR & Epochs & Augmentation\\
\midrule
PRIME & MiniLM & 256 & 24 & 3 & 0.1 & 0.001 & 15 & VN\\
 & ada-002 & 512 & 24 & 3 & 0 & 0.001 & 15 & VN+E+H\\
 & Qwen & 512 & 24 & 3 & 0 & 0.001 & 15 & E+H\\
\midrule
MAG & MiniLM & 256 & 32 & 4 & 0.1 & 0.001 & 15 & VN+E+H\\
 & ada-002 & 512 & 32 & 4 & 0.1 & 0.001 & 15 & VN+E+H\\
 & Qwen & 512 & 32 & 4 & 0.1 & 0.001 & 15 & VN+E+H\\
\midrule
AMAZON & MiniLM & 256 & 24 & 3 & 0.1 & 0.001 & 15 & VN\\
 & ada-002 & 512 & 64 & 3 & 0.1 & 0.0003 & 10 & VN+E+H\\
 & Qwen & 512 & 64 & 3 & 0.1 & 0.0003 & 10 & VN+E+H\\
\bottomrule
\end{tabular}
\end{table}

All main-table runs use AdamW with weight decay $10^{-5}$, gradient-norm clipping at 1, training batch size 16, and evaluation batch size 64. The policy uses eight trajectories per query, temperature 1, recall reward with scale 1, and a mean-reward baseline. The auxiliary BPR loss has weight 1 and is active from the start; the entropy coefficient is zero. Learning rates use linear warmup for the first 10\% of training steps, followed by a cosine schedule. Frozen features and candidate graphs are reused across the five training seeds. MiniLM uses 256-dimensional features and ada-002 and Qwen use 512-dimensional features.

\subsection{Candidate construction and baselines}
Candidate budgets count real nodes, including seeds, and exclude the query virtual node. Entity matches and cosine retrieval provide the initial seeds; hop expansion retains previously unseen nodes under the total cap. \Cref{tab:current_retrieval} lists the stored hop quotas. Quotas are upper bounds: duplicate removal, eligible neighborhoods, and the total cap can reduce the final node count. On the full test splits, mean candidate counts range from 470.34 to 497.83 across the nine configurations.

\begin{table}[t]
\centering\small
\caption{Candidate construction for the main-table SeedER configurations. All graphs have at most 500 real nodes and all outputs have at most 20 eligible nodes.}
\label{tab:current_retrieval}
\begin{tabular}{lll l}
\toprule
Dataset & Encoder & Seed rule & Hop quotas\\
\midrule
PRIME & MiniLM & 5 total & 200, 200, 100\\
 & ada-002 & 5 total & 75, 124, 47, 249\\
 & Qwen & 5 total & 150, 250, 95\\
MAG & All three & 5 total & 99, 149, 148, 99\\
AMAZON & MiniLM & 5 + 10 cosine & 150, 250, 100\\
 & ada-002, Qwen & 20 total & 72, 121, 46, 241\\
\bottomrule
\end{tabular}
\end{table}

The dataset-specific MAG configurations and AMAZON ada-002/Qwen configurations assign entity-derived seeds a first-hop source score of 1. Expander edges and virtual-node edges are distinguished from original knowledge-graph edges. Configurations with gated hop embeddings use the retrieval depth as a positional feature. Encoder comparisons also change preprocessing, candidate graphs, and model settings.

\subsection{Scaling, coverage, and ablations}
The PRIME scaling study uses budgets 50, 100, 200, 400, 800, and 1,600, with five training seeds at every point and an output cap of 20. All three encoders use hidden size 24, three layers, dropout 0, learning rate 0.001, and 15 epochs, with a virtual node, five expander cycles, and gated hop embeddings. Five seeds are included in each budget, and the remaining budget is distributed proportionally across four hops using weights $(150,250,95,500)$. Dashed GraphFlow values are reference results, not measurements at the corresponding candidate budgets.

For the coverage plot, oracle Hit@any measures whether the retrieved eligible candidate set contains any answer, and oracle Recall@any measures the fraction of answers it contains. Achieved Hit@20 and Recall@20 are computed from the selected outputs of the same main-table configurations. Candidate graphs are fixed across training seeds; uncertainty bars apply to the achieved metrics. The oracle does not impose the output cap, so its gap to achieved recall combines ranking limitations with the restriction to 20 outputs.

\subsection{Latency and memory protocol}
We measure 20 fixed PRIME validation queries with three warmup queries and three measured repeats per query, giving 60 measurements per workload. Both workloads use an H100 MIG 40\,GB slice with four CPU cores and 64\,GB host memory. Latency is reported as the mean and sample standard deviation over these measurements; memory is peak allocated GPU memory, including resident model weights.

The SeedER timing uses the Qwen model at 500 candidates, with 1,115,428 trainable parameters and batch size 1. It includes candidate retrieval, graph collation, and reranking, averaging 53.07, 6.80, and 6.97\,ms respectively. Query embeddings are precomputed; embedding generation, PCA fitting, and one-time loading are excluded. 

The Llama3.1-8B-Instruct proxy makes three generative calls along a random-neighbor traversal, using preassembled query and neighborhood text, BF16 weights, SDPA attention, at most 2,048 input tokens, and at most 16 new tokens per call. Timing includes tokenization, device transfer, generation, and decoding, but excludes model loading and text preparation. Its generated responses do not determine the next neighbor. This workload is not an implementation of GraphFlow's trained policy, multi-trajectory search, or reranking, and the measured latency ratio is not a GraphFlow speedup.

\end{document}